\documentclass{article}

\PassOptionsToPackage{numbers, compress}{natbib}
\usepackage{iclr2021_conference,times}
\iclrfinaltrue

\usepackage[utf8]{inputenc} 
\usepackage[T1]{fontenc}    
\usepackage{hyperref}       
\usepackage{url}            
\usepackage{booktabs}       
\usepackage{amsfonts}       
\usepackage{nicefrac}       
\usepackage{microtype}      
\usepackage{graphicx}
\usepackage{amsmath}
\usepackage{mathtools}
\usepackage{amssymb}
\usepackage{amsthm}
\usepackage{cases}
\usepackage{listings}
\usepackage{color}
\usepackage{textcomp}
\usepackage{subfig}
\usepackage{xcolor}
\usepackage{tikz}
\usepackage{caption}
\usepackage{algorithmicx}
\usepackage{algorithm}
\usepackage{algpseudocode}
\usepackage{bbm}
\usepackage[moderate,mathdisplays=normal]{savetrees}
\usepackage{parskip}
\usetikzlibrary{arrows.meta,shapes}

\def\shownotes{0}
\ifnum\shownotes=1
\newcommand{\authnote}[2]{[#1: #2]}
\else
\newcommand{\authnote}[2]{}
\fi

\newcommand{\ak}[1]{{\color{green}\authnote{AK}{#1}}}

\title{In-N-Out: Pre-Training and Self-Training using Auxiliary Information for Out-of-Distribution Robustness}
\author{%
  Sang Michael Xie\thanks{Equal contribution.}, ~Ananya Kumar\footnotemark[1],
  ~Robbie Jones\footnotemark[1], ~Fereshte Khani, ~Tengyu Ma, ~Percy Liang \\
  Stanford University\\
  \texttt{\{xie, ananya, rmjones, fereshte, tengyuma, pliang\}@cs.stanford.edu}
}
\date{}

\newtheorem{remark}{Remark}

\newtheorem{example}{Example}

\newcommand{\probsetting}{\mathcal{S}}
\newcommand{\Routood}{R_\text{out}^\text{ood}}
\newcommand{\Rinnoutood}{R_{\text{in-out}}^{\text{ood}}}

\newcommand{\Rid}{R_\text{id}}

\newcommand{\Rood}{R_\text{ood}}
\newcommand{\Pid}{P_\text{id}}
\newcommand{\Pood}{P_\text{ood}}
\newcommand{\xidi}{{x_i^\text{id}}}
\newcommand{\zidi}{{z_i^\text{id}}}
\newcommand{\xoodi}{{x_i^\text{ood}}}
\newcommand{\zoodi}{{z_i^\text{ood}}}
\newcommand{\laux}{{\ell_\text{aux}}}

\newcommand{\fbs}{{\hat{f}_\text{bs}}}

\newcommand{\fin}{{\hat{f}_\text{in}}}

\newcommand{\Rpre}{{\hat{R}_\text{pre}}}
\newcommand{\Rtrans}{{\hat{R}_\text{trans}}}
\newcommand{\hout}{{\hat{h}_\text{out}}}
\newcommand{\fout}{{\hat{f}_\text{out}}}
\newcommand{\finnout}{\hat{f}}
\newcommand{\ginnout}{\hat{g}}
\newcommand{\Rst}{{\hat{R}_\text{st}}}
\newcommand{\Rstpop}{{R_\text{st}}}
\newcommand{\EE}{{EE}}

\newcommand{\Xz}{{X_Z}}
\newcommand{\Wu}{{W_U}}

\newcommand{\gin}{{\hat{g}_\text{in}}}
\newcommand{\gz}{{g_{\text{z-out}}}}
\newcommand{\gy}{{g_{\text{y-out}}}}
\newcommand{\hatgy}{{\hat g_{\text{y-out}}}}

\newcommand{\midsize}{{m_\text{id}}}
\newcommand{\moodsize}{{m_\text{ood}}}

\newcommand{\vspan}{\mbox{span}}

\newcommand{\xtest}{x_{te}}

\newcommand{\hatA}{\hat{A}}
\newcommand{\hatB}{\hat{B}}

\newcommand{\Astar}{{A^\star}}
\newcommand{\Bstar}{{B^\star}}
\newcommand{\Cstar}{{C^\star}}

\newcommand{\thetastar}{{\theta^{\star}}}

\newcommand{\gammaz}{\gamma_z}
\newcommand{\gammaw}{\gamma_w}
\newcommand{\tx}{\theta_x}
\newcommand{\tz}{\theta_z}
\newcommand{\tw}{\theta_w}
\newcommand{\tu}{\theta_u}
\newcommand{\txbaseline}{{\hat{\theta}_{{x, ols}}}}
\newcommand{\txinput}{{\hat{\theta}_{{x,in}}}}
\newcommand{\twoutput}{\hat{\theta}_{w, out}}
\newcommand{\twinnout}{\hat{\theta}_{w}}
\newcommand{\tzinput}{{\hat{\theta}_{{z,in}}}}

\newcommand{\sigmau}{\sigma_u}

\newcommand{\eigenmax}{\tau_{max}}
\newcommand{\eigenmin}{\tau_{min}}

\newcommand{\singularmin}{\tau_{min}}
\newcommand{\minv}{{-1}}
\DeclareMathOperator*{\argmin}{arg\,min}

\newcommand{\dotp}[2]{\ensuremath{#1^{\top} #2}}

\newlength{\widebarargwidth}
\newlength{\widebarargheight}
\newlength{\widebarargdepth}

\makeatletter
\newenvironment{btHighlight}[1][]
{\begingroup\tikzset{bt@Highlight@par/.style={#1}}\begin{lrbox}{\@tempboxa}}
{\end{lrbox}\bt@HL@box[bt@Highlight@par]{\@tempboxa}\endgroup}

\newcommand\btHL[1][]{%
  \begin{btHighlight}[#1]\bgroup\aftergroup\bt@HL@endenv%
}
\def\bt@HL@endenv{%
  \end{btHighlight}%
  \egroup
}
\newcommand{\bt@HL@box}[2][]{%
  \tikz[#1]{%
    \pgfpathrectangle{\pgfpoint{1pt}{0pt}}{\pgfpoint{\wd #2}{\ht #2}}%
    \pgfusepath{use as bounding box}%
    \node[anchor=base west, fill=orange!30,outer sep=0pt,inner xsep=1pt, inner ysep=0pt, rounded corners=3pt, minimum height=\ht\strutbox+1pt,#1]{\raisebox{1pt}{\strut}\strut\usebox{#2}};
  }%
}
\makeatother
\definecolor{listinggray}{gray}{0.9}
\definecolor{lbcolor}{rgb}{0.9,0.9,0.9}
\newcommand\BR{\ensuremath{\mathbb{R}}}

\newcommand\R{\ensuremath{\mathbb{R}}} 
\newcommand\refeqn[1]{(\ref{eqn:#1})}

\newcommand\refsec[1]{Section~\ref{sec:#1}}

\newcommand\reffig[1]{Figure~\ref{fig:#1}}

\ifthenelse{\isundefined{\definition}}{\newtheorem{definition}{Definition}}{}
\ifthenelse{\isundefined{\assumption}}{\newtheorem{assumption}{Assumption}}{}
\ifthenelse{\isundefined{\hypothesis}}{\newtheorem{hypothesis}{Hypothesis}}{}
\ifthenelse{\isundefined{\proposition}}{\newtheorem{proposition}{Proposition}}{}
\ifthenelse{\isundefined{\theorem}}{\newtheorem{theorem}{Theorem}}{}
\ifthenelse{\isundefined{\lemma}}{\newtheorem{lemma}{Lemma}}{}
\ifthenelse{\isundefined{\corollary}}{\newtheorem{corollary}{Corollary}}{}
\ifthenelse{\isundefined{\alg}}{\newtheorem{alg}{Algorithm}}{}
\ifthenelse{\isundefined{\example}}{\newtheorem{example}{Example}}{}
\newcommand{\E}{\ensuremath{\mathop{\mathbb{E}}}} 

\begin{document}

\maketitle

\begin{abstract}
Consider a prediction setting with few in-distribution labeled examples and
many unlabeled examples both in- and out-of-distribution (OOD).
The goal is to learn a model which performs well both in-distribution and OOD.
In these settings, auxiliary information is often cheaply available for every input.
How should we best leverage this auxiliary information for the prediction task?
Empirically across three image and time-series datasets, and theoretically in a multi-task
linear regression setting, we show that (i) using auxiliary information as
input features improves in-distribution error but can hurt OOD error; but (ii)
using auxiliary information as outputs of auxiliary pre-training tasks improves
OOD error. To get the best of both worlds, we introduce In-N-Out, which first
trains a model with auxiliary inputs and uses it to pseudolabel all the
in-distribution inputs, then pre-trains a model on OOD auxiliary outputs and
fine-tunes this model with the pseudolabels (self-training). We show both
theoretically and empirically that In-N-Out outperforms auxiliary inputs or
outputs alone on both in-distribution and OOD error.
\end{abstract}

\section{Introduction}\label{sec:intro}

When models are tested on distributions that are different from the training
distribution, they typically suffer large drops in
performance~\citep{blitzer2007adaptation, szegedy2014intriguing,jia2017adversarial,albadawy2018tumor,hendrycks2019pretraining}.
For example, in remote sensing, central tasks include
predicting poverty, crop type, and land cover from satellite imagery for
downstream humanitarian, policy, and environmental
applications~\citep{xie2016transfer,jean2016combining,wang2020weakly,russwurm2020meta}.
In some developing African countries, labels are scarce due to the lack of economic resources to deploy human workers to
conduct expensive surveys~\citep{jean2016combining}.
To make accurate predictions in these countries, we must extrapolate to
out-of-distribution (OOD) examples across different geographic terrains and
political borders.

We consider a semi-supervised setting with few in-distribution labeled examples and many unlabeled examples from both in- and out-of-distribution (e.g., global satellite imagery).
While labels are scarce, auxiliary information is often cheaply available for every input and may provide some signal for the missing labels.
Auxiliary information can come from additional data sources (e.g., climate data from other satellites) or derived from the original input (e.g., background or non-visible spectrum image channels).
This auxiliary information is often discarded or not leveraged, and how to best use them is unclear.
One way is to use them directly as input features (\textbf{aux-inputs}); another is to
treat them as prediction outputs for an auxiliary task (\textbf{aux-outputs}) in pre-training.
Which approach leads to better in-distribution or OOD performance?

Aux-inputs provide more features to potentially improve in-distribution performance,
and one may hope that this also improves OOD performance. Indeed, previous
results on standard datasets show that
improvements in in-distribution accuracy correlate with improvements in OOD accuracy~\citep{recht2019doimagenet,taori2020measuring,xie2020selftraining,santurkar2020breeds}.
However, in this paper we find that aux-inputs can introduce more spurious correlations with the labels: as a result, while aux-inputs often improve in-distribution
accuracy, they can worsen OOD accuracy.
We give examples of this trend on CelebA~\citep{liu2015deep} and real-world satellite datasets in Sections~\ref{sec:main-results} and \ref{sec:input-choice}.

\begin{figure}
\begin{minipage}{.75\textwidth}
		\tikzset{
	-latex,auto,node distance =1 cm and 1 cm,semithick,
		state/.style ={circle, draw, minimum width = 5em, minimum height=5em, inner sep=1pt,text width=4em,text centered,},
		fstate/.style ={minimum width = 0.5 cm, inner sep=1pt, text width=9em,text centered,},
		point/.style = {circle, draw, inner sep=0.04cm,fill,node contents={}},
		bidirected/.style={latex-latex,dashed},
		el/.style = {inner sep=2pt, align=left, sloped},
	}
\centering
\scalebox{0.65}{
	\begin{tikzpicture}
	\def \w {4.5};
	\def \h {5}

	\node[state] (aux-in) at (0,0) {Aux-in model};
	\node[fstate] (x-in1)  at (-\h/4,-\w/2){Data\\ ($X$)};
	\node[fstate] (z-in1)  at (+\h/4,-\w/2){Auxiliary info\\ ($Z$)};
	\node[fstate] (y-out1)  at (0,+\w/2){Labels\\ ($Y$)};
	\draw[-{Latex[length=3mm]}] (z-in1) -- (aux-in);
	\draw[-{Latex[length=3mm]}] (x-in1) -- (aux-in);
	\draw[-{Latex[length=3mm]}] (aux-in) -- (y-out1);

	\node[state] (aux-out) at (\h,0) {Aux-out model};
        \node[fstate] (x-in2)  at (\h,-\w/2){Unlabeled data\\ ($X_{\text{un}}$)};
        \node[fstate] (z-out2)  at (\h,+\w/2){Unlabeled aux\\ ($Z_{\text{un}}$)};
		\draw[-{Latex[length=3mm]}] (z-in1) -- (aux-in);
	\draw[-{Latex[length=3mm]}] (x-in2) -- (aux-out);
	\draw[-{Latex[length=3mm]}] (aux-out) -- (z-out2);

	\node[state] (final-model) at (2*\h,0) {In-N-Out model};
	\node[fstate] (x-in1)  at (2*\h-\h/4,-\w/2){Data\\ ($X$)};
        \node[fstate] (x-in2)  at (2*\h+\h/4,-\w/2){In-distribution unlabeled data\\ ($X_{\text{un}}^{\text{id}}$)};
	\node[fstate] (y-out1)  at (2*\h - \h/4,+\w/2){Labels\\ ($Y$)};
        \node[fstate] (y-out2)  at (2*\h + \h/4,+\w/2){Pseudolabels\\ \small(Aux-in($X_{\text{un}}^{\text{id}},Z_{\text{un}}^{\text{id}}$))};
	\draw[-{Latex[length=3mm]}] (x-in1) -- (final-model);
	\draw[-{Latex[length=3mm]}] (x-in2) -- (final-model);
	\draw[-{Latex[length=3mm]}] (final-model) -- (y-out1);
	\draw[-{Latex[length=3mm]}] (final-model) -- (y-out2);

        \draw[dashed,-{Latex[length=3mm]}] (aux-out) -- (final-model) node[midway] {initialize};

        \node[el] at (0,-3*\w/4) {Step 1: Training (aux-in)};
        \node[el] at (\h,-3*\w/4) {Step 2: Pre-training (aux-out)};
	\node[el] at (2*\h,-3*\w/4) {Step 3: Fine-tuning \& self-training};
	\end{tikzpicture}}
\caption{\label{fig:in-n-out} A sketch of the In-N-Out algorithm which consists of three steps:
    1) use auxiliary information as input (Aux-in) to achieve good in-distribution performance,
    2) use auxiliary information as output in pre-training (Aux-out), to improve OOD performance,
	3) fine-tune the pretrained model from step 2 using the labeled data and in-distribution unlabeled data with pseudolabels generated from step 1 to improve in- and out-of-distribution.
  }
\end{minipage}
\hfill
\begin{minipage}{.23\textwidth}
  \centering
  \includegraphics[width=0.85\textwidth]{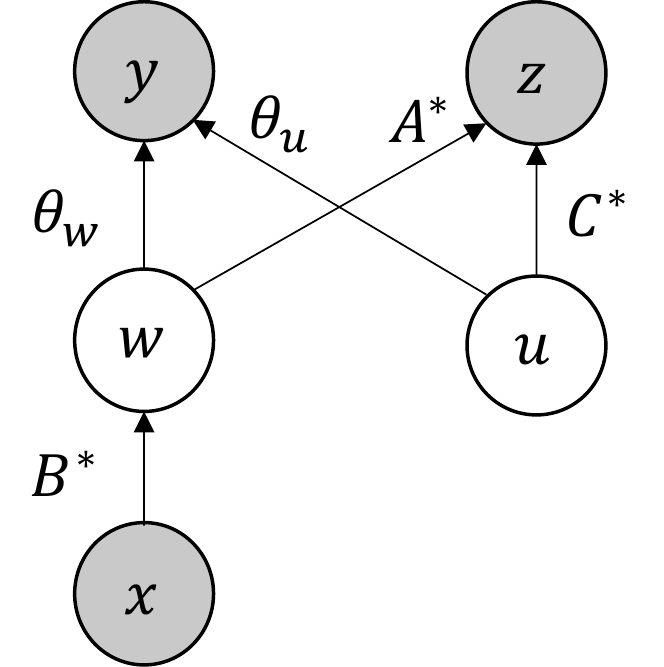}
  \caption{
    Graphical model for our theoretical setting: prediction task with input $x$, target $y$, and auxiliary information $z$, which is related to $y$ through the latent variable $w$ and latent noise $u$.
    }\label{fig:graphical-model}
\end{minipage}

\end{figure}

Conversely, aux-output methods such as pre-training may improve OOD performance
through auxiliary supervision~\citep{caruana97multitask,weiss2016survey,hendrycks2019pretraining}.
\citet{hendrycks2019pretraining} show that pre-training on ImageNet can improve
adversarial robustness, and~\citet{hendrycks2019selfsupervised} show that auxiliary self-supervision tasks can improve robustness to synthetic corruptions.
In this paper, we find that while aux-outputs improve OOD accuracy, the
in-distribution accuracy is worse than with aux-inputs. Thus, we elucidate a
tradeoff between in- and out-of-distribution accuracy that occurs when using auxiliary
information as inputs or outputs.

To theoretically study how to best use auxiliary information, we extend the
multi-task linear regression setting~\citep{du2020fewshot,tripuraneni2020multitask} to allow for distribution
shifts. We show that auxiliary information helps in-distribution error by providing useful features for predicting the target, but the relationship between the aux-inputs and the target can shift significantly OOD, worsening the OOD error. In contrast, the aux-outputs model first pre-trains on unlabeled data to learn a lower-dimensional representation and then solves the target task in the lower-dimensional space.
We prove that the aux-outputs model improves robustness to \emph{arbitrary} covariate shift compared to not using auxiliary information.

Can we do better than using auxiliary information as inputs or outputs alone?
We answer affirmatively by proposing the In-N-Out algorithm to combine the benefits of auxiliary inputs and outputs (Figure~\ref{fig:in-n-out}).
In-N-Out first uses an aux-inputs model, which has good in-distribution
accuracy, to pseudolabel in-distribution unlabeled data.
It then pre-trains a model using aux-outputs and finally fine-tunes this model
on the larger training set consisting of labeled and pseudolabeled data.
We prove that In-N-Out, which combines self-training and pre-training,
further improves both in-distribution and OOD error over the aux-outputs model.

We show empirical results on CelebA and two remote sensing tasks (land cover and cropland prediction) that parallel the theory.
On all datasets, In-N-Out improves OOD
accuracy and has competitive or better in-distribution accuracy over aux-inputs
or aux-outputs alone and improves 1--2\% in-distribution, 2--3\% OOD over not using auxiliary information on remote sensing tasks.
Ablations of In-N-Out show that In-N-Out achieves
similar improvements over pre-training or self-training alone (up to 5\% in-distribution, 1--2\% OOD on remote sensing tasks).
We also find that using OOD (rather than in-distribution) unlabeled examples for pre-training is crucial for OOD improvements.

\section{Setup}
\label{sec:setup}
Let $x \in \R^d$ be the input (e.g., a satellite image), $y \in \R$ be the target (e.g., crop type), and $z \in \R^T$ be the cheaply obtained auxiliary information either from additional sources (e.g., climate information) or derived from the original data (e.g., background).

\paragraph{Training data.} Let $\Pid$ and $\Pood$ denote the underlying distribution of $(x,y,z)$ triples in-distribution and out-of-distribution, respectively.
The training data consists of
(i) in-distribution labeled data $\{(x_i, y_i, z_i)\}_{i=1}^n \sim \Pid$,
(ii) in-distribution unlabeled data $\{(\xidi, \zidi)\}_{i=1}^{\midsize} \sim \Pid$, and 
(iii) out-of-distribution unlabeled data  $\{(\xoodi, \zoodi)\}_{i=1}^{\moodsize} \sim \Pood$. 

\paragraph{Goal and risk metrics.} Our goal is to learn a model from input and auxiliary information to the target, $f: \R^d \times \R^T \rightarrow \R$.
For a loss function $\ell$, the in-distribution population risk of the model $f$ is $\Rid(f) = \E_{x, y, z \sim \Pid}[ \ell( f(x, z), y ) ]$, and its OOD population risk is $\Rood(f) = \E_{x, y, z \sim \Pood}[ \ell( f(x, z), y ) ]$.

\subsection{Models}

We consider three common ways to use the auxiliary information ($z$) to learn a model.

\paragraph{Baseline.} The baseline minimizes the empirical risk on labeled data while ignoring the auxiliary information (accomplished by setting $z$ to 0):
\begin{align}
\label{eqn:base}
\fbs = \argmin_{f} \frac{1}{n} \sum_{i=1}^n \ell( f(x_i,0), y_i).
\end{align}

\paragraph{Aux-inputs.} The aux-inputs model minimizes the empirical risk on labeled data while using the auxiliary information as features:
\begin{align}
\label{eqn:aux-in}
	\fin = \argmin_{f} \frac{1}{n} \sum_{i=1}^n \ell( f(x_i, z_i), y_i).
\end{align}

\paragraph{Aux-outputs.} The aux-outputs model leverages the auxiliary information $z$ by using it as the prediction target of an auxiliary task,
in hopes that there is a low-dimensional feature representation that is common to predicting both $z$ and $y$.
Training the aux-outputs model consists of two steps:

In the \emph{pre-training} step, we use all the unlabeled data to learn a shared feature representation.
Let $h: \R^d \to \R^k$ denote a feature map and $\gz: \R^k \to \R^T$ denote a mapping from feature representation to the auxiliary outputs. Let $\laux$ denote the loss function for the auxiliary information.
We define the empirical risk of $h$ and $\gz$ as:
\begin{align}
\label{eqn:pretrain-aux-output-finite-data-empirical}
\Rpre(h, \gz) = \frac{1}{\midsize+\moodsize} \Big( \sum_{i=1}^{\midsize} \laux( \gz(h(\xidi)), \zidi) + \sum_{i=1}^{\moodsize} \laux( \gz(h(\xoodi)), \zoodi) \Big).
\end{align}
The estimate of the feature map is $\hout = \argmin_{h} \min_{\gz} \Rpre(h, \gz)$.

In the \emph{transfer} step, the model uses the pre-trained feature map $\hout$ and the labeled data to learn the mapping $\gy : \R^k \to \R$ from feature representation to target $y$.
We define the transfer empirical risk as:
\begin{align}
\Rtrans(\hout, \gy) = \frac{1}{n} \sum_{i=1}^n \ell( \gy(\hout(x_i)), y_i)
\end{align} 
The estimate of the target mapping is $\hatgy = \argmin_{\gy} \Rtrans(\hout, \gy)$. The final aux-outputs model is
\begin{align}
\label{eqn:aux-out}
    \fout(x, z) = \hatgy(\hout(x)).
\end{align}
Like the baseline model, the aux-outputs model ignores the auxiliary information for prediction.

\section{Theoretical Analysis of Aux-inputs and Aux-outputs Models}
\label{sec:input-output-theory}
We now analyze the baseline, aux-inputs, and aux-outputs models introduced in \refsec{setup}.
Our setup extends a linear regression setting commonly used for analyzing multi-task problems~\citep{du2020fewshot,tripuraneni2020multitask}.

\textbf{Setup.} See \reffig{graphical-model} for the graphical model.
Let $w = \Bstar x \in \R^k$ be a low-dimensional latent feature $(k\le d)$
shared between auxiliary information $z$ and the target $y$.
Let $u \in \R^m$ denote unobserved latent variables not captured in $x$.
We assume $z$ and $y$ are linear functions of $u$ and $w$:
\begin{align}
y &= \dotp{\tw}{w} + \dotp{\tu}{u} + \epsilon,\\
z &= \Astar w + \Cstar u,
\end{align}
where $\epsilon \sim P_{\epsilon}$ denotes noise with mean $0$ and variance $\sigma^2$.
As in~\cite{du2020fewshot}, we assume the dimension of the auxiliary information $T$ is greater than the feature dimension $k$, that is $T \geq k$, and that $\Astar, \Bstar$ and $\Cstar$ have full rank (rank $k$). We also assume $T \geq m$, where $m$ is the dimension of $u$.
  
\textbf{Data.}
Let $P_\text{x}$ and $P_\text{u}$ denote the distribution of $x$ and $u$ in-distribution (ID), and
let $P_\text{x}'$, $P_\text{u}'$ denote the distribution $x$ and $u$ OOD.
We assume $x$ and $u$ are independent, have distributions with bounded density everywhere, and have invertible covariance matrices.
We assume the mean of $u$ is zero in- and out-of-distribution\footnote{This is not limiting because bias in $z$ can be folded into $x$.}.
We assume we have $n \geq m+d$ in-distribution labeled training examples and unlimited access to unlabeled data both ID and OOD, a common assumption in unsupervised domain adaptation theory~\citep{sugiyama2007covariate,kumar2020gradual,raghunathan2020understanding}.

\textbf{Loss metrics.} We use the squared loss for the target and auxiliary losses: $\ell (\hat y, y) = (y- \hat y)^2$ and $\laux (z,z') = \|z - z'\|^2_2$.

\textbf{Models.}
We assume all model families ($f$, $h$, $\gz$, $\gy$) in Section~\ref{sec:setup} are linear.

Let $\probsetting = (\Astar, \Bstar, \Cstar, \tw, \tu, P_\text{x}, P_\text{u})$ denote a problem setting which satisfies all the above assumptions.

\subsection{Auxiliary inputs help in-distribution, but can hurt OOD}

We first show that the aux-inputs model \refeqn{aux-in} performs better than the baseline model \refeqn{base} in-distribution.
Intuitively, the target $y$ depends on both the inputs $x$ (through $w$) and latent variable $u$ (Figure~\ref{fig:graphical-model}).
The baseline model only uses $x$ to predict $y$; thus it cannot capture the variation in $y$ due to $u$.
On the other hand,
the aux-inputs model uses $x$ and $z$ to predict $y$.
Since $z$ is a function of $x$ (through $w$) and $u$, $u$ can be recovered from $x$ and $z$ by inverting this relation.
Note that $u$ is unobserved but implicitly recovered.
The aux-inputs model can then combine $u$ and $x$ to predict $y$ better.

Let $\sigmau^2 = \E_{u \sim P_u}[(\dotp{\tu}{u})^2]$ denote the (in-distribution) variance of $y$ due to the latent variables $u$.
The following proposition shows that if $\sigmau^2 > 0$ then with enough training examples the aux-inputs model has lower in-distribution population risk than the baseline model.\footnote{Since $z$ is typically low-dimensional and $x$ is high-dimensional (e.g., images), the aux-inputs model needs only a slightly larger number of examples before it outperforms the baseline.}

\newcommand{\auxInputInDomainText}{
  For all problem settings $\probsetting$, $P_{\epsilon}$, assuming regularity conditions (bounded $x, u$, sub-Gaussian noise $\epsilon$, and $T = m$), and $\sigmau^2 > 0$,
for all $\delta > 0$,
there exists $N$ such that for $n \geq N$ number of training points,
with probability at least $1 - \delta$ over the training examples,
the aux-inputs model improves over the baseline:
\begin{equation}
  \Rid(\fin) < \Rid(\fbs).
\end{equation}
}
\begin{proposition}
\label{prop:auxInputInDomain}
\auxInputInDomainText{}
\end{proposition}

Although using $z$ as input leads to better in-distribution performance, we show that the aux-inputs model can perform worse than the baseline model OOD for any number of training examples.
Intuitively, the aux-inputs model uses $z$, which can be unreliable OOD because $z$ depends on $u$ and $u$ can shift OOD.
In more detail, the aux-inputs model learns to predict $\hat{y} = \dotp{\txinput}{x} + \dotp{\tzinput}{z}$, where the true output $y = \dotp{\tx}{x} + \dotp{\tz}{z}$, and $\tzinput$ is an approximation to the true parameter $\tz$, that has some error.
Out-of-distribution $u$ and hence $z$ can have very high variance, which would magnify $(\tzinput - \tz)^\top z$ and lead to bad predictions.

\newcommand{\auxInputOODExampleText}{
There exists a problem setting $\probsetting$, $P_{\epsilon}$, such that for every $n$, there is some test distribution $P_x', P_u'$ with:
\begin{equation}
\E[\Rood(\fin)] > \E[\Rood(\fbs)]
\end{equation}
}

\begin{example}
\label{ex:input-model-bad}
\auxInputOODExampleText{}
\end{example}

\subsection{Pre-training improves risk under arbitrary covariate shift}

While using $z$ as inputs (aux-inputs) can worsen performance relative to the baseline, our first main result is that the aux-outputs model (which pre-trains to predict $z$ from $x$, and then transfers the learned representation to predict $y$ from $x$) outperforms the baseline model for all test distributions.

\textbf{Intuition.}
Referring to Figure~\ref{fig:graphical-model}, we see that the mapping from inputs $x$ to auxiliary $z$ passes through the lower dimensional features $w$.
In the pre-training step, the aux-outputs model predicts $z$ from $x$ using a low rank linear model, and we show that this recovers the `bottleneck' features $w$ (up to symmetries; more formally we recover the rowspace of $\Bstar$).
In the transfer step, the aux-outputs model learns a linear map from the lower-dimensional $w$ to $y$, while the baseline predicts $y$ directly from $x$.
To warm up, \emph{without distribution shift}, the expected excess risk only depends on the dimension of the input, and not the conditioning.
That is, the expected excess risk in linear regression is exactly $d \sigma^2 / n$, where $d$ is the input dimension, so the aux-outputs trivially improves over the baseline since $\mbox{dim}(w) < \mbox{dim}(x)$.
In contrast, the \emph{worst case risk under distribution shift depends on the conditioning of the data}, which could be worse for $w$ than $x$.
Our proof shows that the worst case risk (over all $x$ and $u$) is still better for the aux-outputs model because projecting to the low-dimensional feature representation ``zeroes-out'' some error directions.

\newcommand{\auxOutputHelpsEverywhereText}{
  For all problem settings $\probsetting$, noise distributions $P_{\epsilon}$, test distributions $P_\text{x}'$, $ P_\text{u}'$, and $n \geq m + d$ number of training points:
\begin{equation}
  \E[\Rood(\fout)] \leq \E[\Rood(\fbs)].
\end{equation}
}

\begin{theorem}
\label{thm:auxOutputHelpsEverywhere}
\auxOutputHelpsEverywhereText{}
\end{theorem}

See Appendix~\ref{appendix:proofs} for the proof.

\section{In-N-Out: combining auxiliary inputs and outputs}
\label{sec:in-n-out}
We propose the In-N-Out algorithm, which combines both the aux-inputs and aux-outputs models for further complementary gains (Figure~\ref{fig:in-n-out}).
As a reminder: (i) The aux-inputs model ($x, z \to y$) is good in-distribution, but bad OOD because $z$ can be misleading OOD. (ii) The aux-outputs model ($x \to y$) is better than the baseline OOD, but worse than aux-inputs in-distribution because it doesn’t use $z$. (iii) We propose the In-N-Out model ($x \to y$), which uses pseudolabels from aux-inputs (stronger model) in-distribution to transfer in-distribution accuracy to the aux-outputs model.
The In-N-Out model does not use $z$ to make predictions since $z$ can be misleading / spurious OOD.

In more detail, we use the aux-inputs model (which is good in-distribution) to pseudolabel in-distribution unlabeled data.
The pseudolabeled data provides more effective training samples (self-training) to fine-tune an aux-outputs model pre-trained on predicting auxiliary information from all unlabeled data.
We present the general In-N-Out algorithm in Algorithm~\ref{alg:innout} and analyze it in the linear multi-task regression setting of Section~\ref{sec:setup}.
The In-N-Out model $\finnout=\ginnout \circ \hout$ optimizes the empirical risk on labeled and pseudolabeled data:
\begin{equation}
    \label{eqn:innout_loss}
       \ginnout = \argmin_g (1 - \lambda) \Rtrans(\hout, g) +  \lambda \Rst(\hout, \fin, g)
\end{equation}
where $\Rst(\hout, \fin, g) = \frac{1}{m_1} \sum_{i=1}^{m_1} \ell( g(\hout(\xidi)), \fin(\xidi, \zidi) )$ is the loss of self-training on pseudolabels from the aux-inputs model,
and $\lambda \in [0, 1]$ is a hyperparameter that trades off between labeled and pseudolabeled losses. In our experiments, we fine-tune $\ginnout$ and $\hout$ together.
\begin{algorithm}[tbp]
	\caption{In-N-Out}	\label{alg:innout}
	\begin{algorithmic}[1]
		\small
				\Require  in-distribution labeled data $\{(x_i, y_i, z_i)\}_{i=1}^n \sim \Pid$, 
		\Statex	\hspace{0.7cm} in-distribution unlabeled data $\{(\xidi, \zidi)\}_{i=1}^{\midsize} \sim \Pid$, 
		\Statex \hspace{0.7cm} OOD unlabeled data $\{(\xoodi, \zoodi)\}_{i=1}^{\moodsize} \sim \Pood$
\State	Learn  $\fin : (x, z) \mapsto y$ from  in-distribution labeled data $\{(x_i, y_i, z_i)\}_{i=1}^n \sim \Pid$\;
\State	Pre-train $\gz \circ \hout : x \mapsto z$ on aux-outputs from all unlabeled data $\{(\xidi, \zidi)\}_{i=1}^{\midsize} \cup \{(\xoodi, \zoodi)\}_{i=1}^{\moodsize}$\;
\State Return $\finnout = \ginnout \circ \hout: x \mapsto y$ trained on labeled and pseudolabeled data  $\{(x_i, y_i)\}_{i=1}^n \cup \{(\xidi, \fin(\xidi, \zidi)\}_{i=1}^{\midsize}$
	\end{algorithmic}
\end{algorithm}

\textbf{Theoretical setup.}
Because fine-tuning is difficult to analyze theoretically, we analyze a slightly modified version of In-N-Out where we train an aux-inputs model to predict $y$ given the features $\hout(x)$ and auxiliary information $z$, so the aux-inputs model $\gin : \BR^k \times \BR^T \to \BR$ is given by $\gin = \argmin_g \frac{1}{n} \sum_{i=1}^n \ell( g(\hout(x_i), z_i), y_i)$.
The population self-training loss on pseudolabels from the aux-inputs model $\gin \circ \hout$ is: $\Rstpop(\hout, \gin, g) = \E_{x, z \sim \Pid}[\ell(g(\hout(x)), \gin(\hout(x), z))]$, and we minimize the self-training loss: $\ginnout = \argmin_g \Rstpop(\hout, \gin, g)$.
At test time given input $x, z$ the In-N-Out model predicts $\ginnout(\hout(x))$.
For the theory, we assume all models ($\gin, \ginnout, and \hout$) are linear.

\subsection{In-N-Out improves over pre-training under arbitrary covariate shift}

We prove that In-N-Out helps on top of pre-training, as long as the auxiliary features give us information about $y$ relative to the noise $\epsilon$ in-distribution---that is, if $\sigmau^2$ is much larger than $\sigma^2$.

To build intuition, first consider the special case where the noise $\sigma^2 = 0$ (equivalently, $\epsilon=0$).
Since $u$ can be recovered from $w$ and $z$, we can write $y$ as a linear function of $w$ and $z$: $y = \dotp{\gammaw}{w} + \dotp{\gammaz}{z}$.
We train an aux-inputs model $\gin$ from $w, z$ to $y$ on finite labeled data.
Since there is no noise, $\gin$ predicts $y$ perfectly from $w, z$ (we learn $\gammaw$ and $\gammaz$).
We use $\gin$ to pseudolabel a large amount of unlabeled data, and since $\gin$ predicts $y$ perfectly from $w, z$, the pseudolabels are perfect.
So here pseudolabeling gives us a much larger and correctly labeled dataset to train the In-N-Out model on.

The technical challenge is proving that self-training helps under arbitrary covariate shift even when the noise is non-zero ($\sigma^2 > 0$), so the aux-inputs model $\gin$ that we learn is accurate but not perfect.
In this case, the pseudolabels have an error which propagates to the In-N-Out model self-trained on these pseudolabels, but we want to show that the error is lower than for the aux-outputs model.
The error in linear regression is proportional to the noise of the target $y$, which for the aux-outputs model is $\sigma^2 + \sigmau^2$.
We show that the In-N-Out model uses the aux-inputs model to reduce the dependence on the noise $\sigmau^2$, because the aux-inputs model uses both $w$ and $z$ to predict $y$.
The proof reduces to showing that the max singular value for the In-N-Out error matrix is less than the min-singular value of the aux-outputs error matrix with high probability.
A core part of the argument is to lower bound the min-singular value of a random matrix (Lemma~\ref{lem:min-sing-lemma}).
This uses techniques from random matrix theory (see e.g., Chapter 2.7 in~\citet{tao2012random}); the high level idea is to show that with probability $1 - \delta$ each column of the random matrix has a (not too small) component orthogonal to all other columns.
 
\newcommand{\selfTrainingHelpsEverywhereText}{
  In the linear setting, for all problem settings $\probsetting$ with $\sigmau^2 > 0$, test distributions $P_x', P_u'$, $n \geq m + d$ number of training points, and $\delta > 0$, there exists $a, b > 0$ such that for all noise distributions $P_{\epsilon}$, with probability at least $1 - \delta$ over the training examples and test example $x' \sim P_x'$, the ratio of the excess risks (for all $\sigma^2$ small enough that $a - b \sigma^2 > 0$) is:
\begin{equation}
\frac{\Rinnoutood - R^*}{\Routood - R^*} \leq \frac{\sigma^2}{a - b \sigma^2}
\end{equation}
Here $R^* = \min_{g^*, h^*} \E_{x', y', z' \sim P'}[ \ell( g^*(h^*(x')), y' ) ]$ is the min. possible (Bayes-optimal) OOD risk, \linebreak
$\Rinnoutood = \E_{y' \sim P_{y'\mid x'}'}[\ell(\ginnout(\hout(x')), y')]$ is the risk of the In-N-Out model on test example $x'$, and $\Routood = \E_{y' \sim P_{y'\mid x'}'}[\ell(\hatgy(\hout(x')), y')]$ is the risk of the aux-outputs model on test example $x'$. Note that $\Rinnoutood$ and $\Routood$ are random variables that depend on the test input $x'$ and the training set $X$.
}

\begin{theorem}
\label{thm:self-train}
\selfTrainingHelpsEverywhereText{}
\end{theorem}

\begin{remark}
  As $\sigma \to 0$, the excess risk ratio of In-N-Out to Aux-outputs goes to $0$, so the In-N-Out estimator is much better than the aux-outputs estimator.
\end{remark}

The proof of the result is in Appendix~\ref{appendix:proofs}.

\section{Experiments}
\label{sec:experiments}
We show on real-world datasets for land cover and cropland prediction that aux-inputs can hurt OOD performance, while aux-outputs improve OOD performance. In-N-Out improves OOD accuracy and has competitive or better in-distribution accuracy over other models on all datasets (Section~\ref{sec:main-results}).
Secondly, we show that the tradeoff between in-distribution and OOD performance depends on the choice of auxiliary information on CelebA and cropland prediction (Section~\ref{sec:input-choice}).
Finally, we show that OOD unlabeled examples are important for improving OOD robustness (Section~\ref{sec:unlabeled-in-out}).

\subsection{Experimental Setup}
\label{sec:datasets}

\begin{figure}[tbp]
\begin{minipage}{.62\textwidth}
\centering
\includegraphics[width=0.99\textwidth]{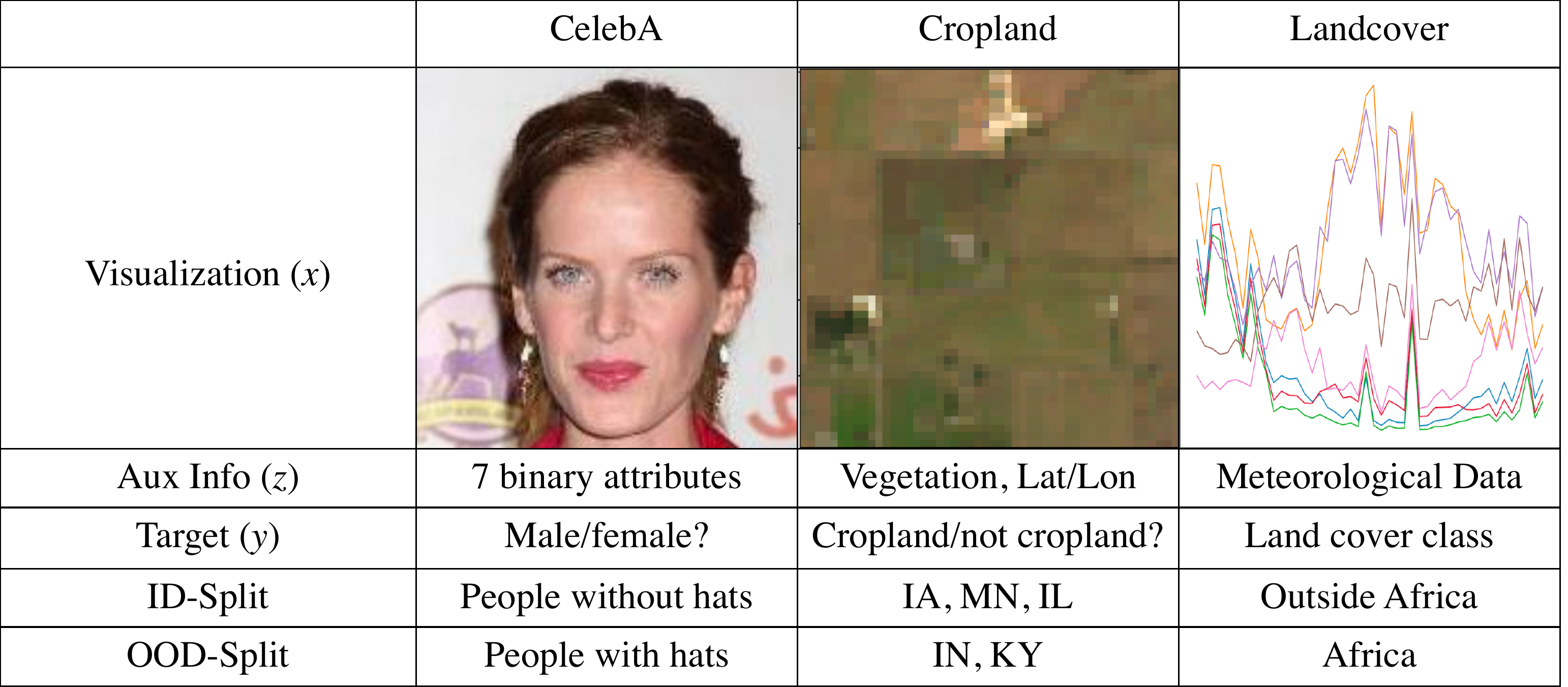}
  \caption{Summary of the datasets used in our experiments.}\label{fig:dataset-table}
\end{minipage}
\hfill
\begin{minipage}{.36\textwidth}
    \centering
  \includegraphics[width=0.8\textwidth]{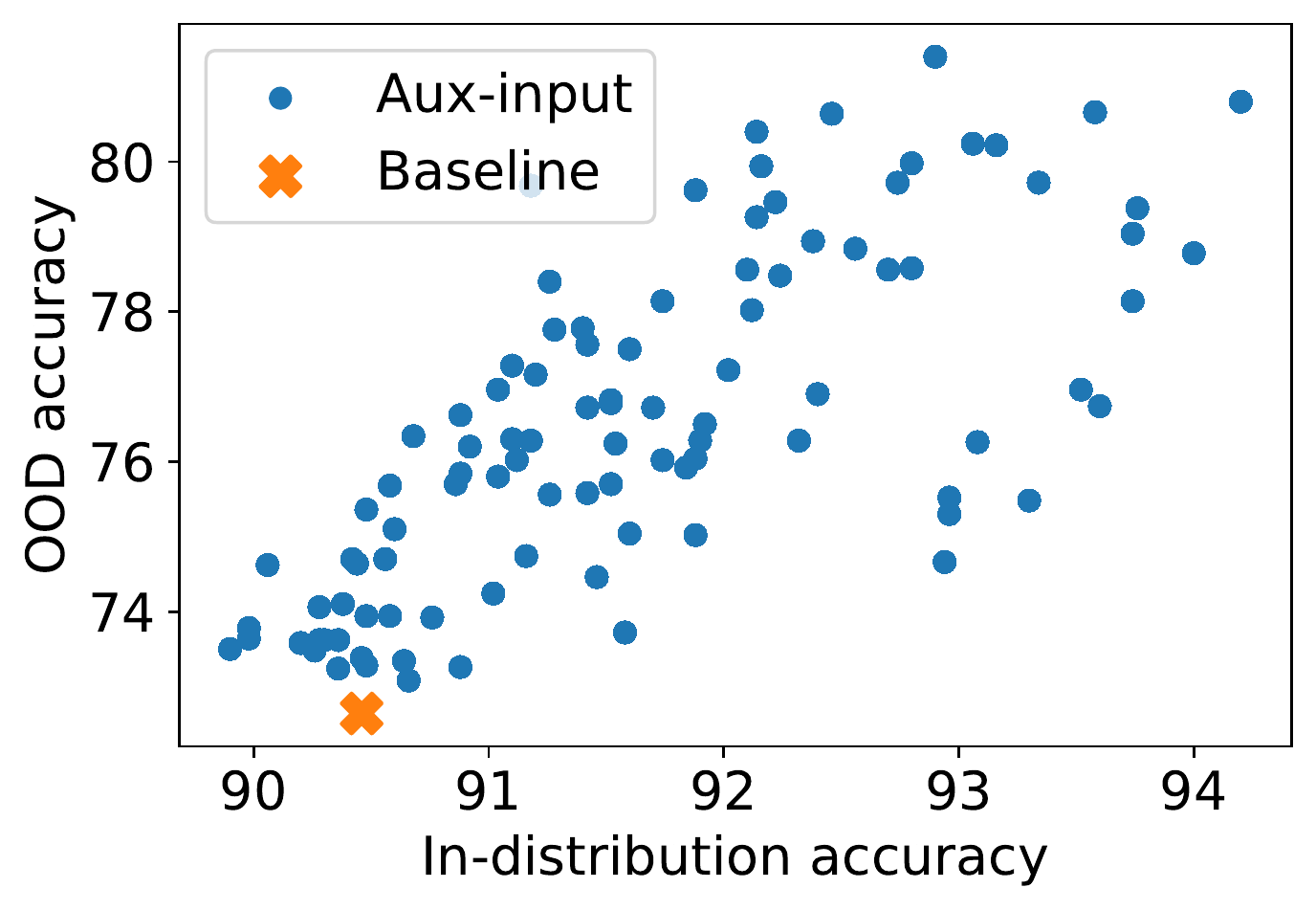}
  \caption{
    Correlation ($r=0.72$) between in-distribution accuracy and OOD accuracy when adding 1 to 15 random auxiliary inputs in CelebA.}\label{fig:celeba-correlation} 
\end{minipage}
\end{figure}

We give a summary of considered datasets and setup here --- see Figure~\ref{fig:dataset-table} and Appendix~\ref{app:experiment-details} for details.
Our datasets use auxiliary information both derived from the input (CelebA, Cropland) and from other sources (Landcover).

\paragraph{CelebA.}
In CelebA~\citep{liu2015deep}, the input $x$ is a RGB image (resized to $64\times 64$), the target $y$ is a binary label for gender, and the auxiliary information $z$ are 7 (of 40) binary-valued attributes derived from the input (e.g., presence of makeup, beard).
We designate the set of images where the celebrity is wearing a hat as OOD.
We use a ResNet18 as the backbone model architecture for all models (see Appendix~\ref{app:celeba} for details).

\paragraph{Cropland.}
Crop type or cropland prediction is an important intermediate problem for crop yield prediction~\citep{cai2018crop, johnson2016crop, kussul2017classification}. The input $x$ is a $50 \times 50$ RGB image taken by a satellite, the target $y$ is a binary label that is 1 when the image contains majority cropland, and the auxiliary information $z$ is the center location coordinate plus $50\times 50$ vegetation-related bands. The vegetation bands in the auxiliary information $z$ is derived from the original satellite image, which contains both RGB and other frequency bands. We use the Cropland dataset from~\citet{wang2020weakly}, with data from the US Midwest. We designate Iowa, Missouri, and Illinois as in-distribution and Indiana and Kentucky as OOD. Following~\citet{wang2020weakly}, we use a U-Net-based model~\citep{ronneberger2015unet}. See Appendix~\ref{app:cropland} for details.

\paragraph{Landcover.}
Land cover prediction involves classifying the land cover type (e.g., “grasslands”) from satellite data at a location~\citep{gislason2006landcover, russwurm2020meta}). The input $x$ is a time series measured by NASA's MODIS satellite~\citep{modis2015landcover}, the target $y$ is one of 6 land cover classes, and the auxiliary information $z$ is climate data (e.g., temperature) from ERA5, a dataset computed from various satellites and weather station data~\citep{dataset2017era5}. We designate non-African locations as in-distribution and Africa as OOD. We use a 1D-CNN to handle the temporal structure in the MODIS data. See Appendix~\ref{app:landcover} for details.

\paragraph{Data splits.} We first split off the OOD data, then split the rest into training, validation, and in-distribution test (see Appendix~\ref{app:experiment-details} for details).
We use a portion of the training set and OOD set as in-distribution and OOD unlabeled data respectively. The rest of the OOD set is held out as test data. We run 5 trials, where we randomly re-generate the training/unlabeled split for each trial (keeping held-out splits fixed).
We use a reduced number of labeled examples from each dataset (1\%, 5\%, 10\% of labeled examples for CelebA, Cropland, and Landcover respectively), with the rest as unlabeled.

\paragraph{Repeated self-training.} In our experiments, we also consider augmenting In-N-Out models with repeated self-training, which has fueled recent improvements in both domain adaptation and ImageNet classification~\citep{shu2018dirtt,xie2020selftraining}. For one additional round of repeated self-training, we use the In-N-Out model to pseudolabel all unlabeled data (both ID and OOD) and also initialize the weights with the In-N-Out model.
Each method is trained with early-stopping and hyperparameters are chosen using the validation set.

\subsection{Main Results}
\label{sec:main-results}
\begin{table}[tbp]
  \centering
\scalebox{0.80}{
\begin{tabular} {l r r r r r r}
\toprule
& \multicolumn{2}{c}{CelebA} & \multicolumn{2}{c}{Cropland} & \multicolumn{2}{c}{Landcover}\\
 & ID Test Acc & OOD Acc & ID Test Acc & OOD Acc & ID Test Acc & OOD Test Acc\\
\midrule
    Baseline & 90.46 $\pm$ 0.85 & 72.64 $\pm$ 1.39 & 94.50 $\pm$ 0.11 & 90.30 $\pm$ 0.75 & 75.92 $\pm$ 0.25 & 58.31 $\pm$ 1.87\\
    Aux-inputs & 92.36 $\pm$ 0.29 & 77.4 $\pm$ 1.33 & \textbf{95.34} $\pm$ 0.22 & 84.15 $\pm$ 4.23 & 76.58 $\pm$ 0.44 & 54.78 $\pm$ 2.01\\
    Aux-outputs & \textbf{94.0} $\pm$ 0.24 & 77.68 $\pm$ 0.59 & 95.12 $\pm$ 0.15 & 91.63 $\pm$ 0.21 & 72.48 $\pm$ 0.37 & 61.03 $\pm$ 0.97 \\
    In-N-Out (no pretrain) & \textbf{93.8} $\pm$ 0.56 & 78.54 $\pm$ 1.31 & 94.93 $\pm$ 0.15 & 91.23 $\pm$ 0.61 & 76.54 $\pm$ 0.23 & 59.19 $\pm$ 0.98 \\
    In-N-Out & 93.42 $\pm$ 0.36 & 79.42 $\pm$ 0.70 & \textbf{95.45} $\pm$ 0.16 & \textbf{91.94} $\pm$ 0.57 & \textbf{77.43} $\pm$ 0.39 & 61.53 $\pm$ 0.74\\
    In-N-Out + repeated ST & \textbf{93.76} $\pm$ 0.46 & \textbf{80.38} $\pm$ 0.68 & \textbf{95.53} $\pm$ 0.19 & \textbf{92.18} $\pm$ 0.40 & \textbf{77.10} $\pm$ 0.30 & \textbf{62.61} $\pm$ 0.58\\
\bottomrule
\end{tabular}
}
\caption{Accuracy (\%) of various models using auxiliary information as input, output, or both. In-N-Out generally improves both in- and out-of-distribution over aux-inputs or aux-outputs alone. Results are averaged over 5 trials with 90\% intervals. Repeated ST refers to one round of repeated self-training on top of In-N-Out.}
\label{table:main}
\end{table}

Table~\ref{table:main} compares the in-distribution (ID) and OOD accuracy of different methods.
In all datasets, pre-training with aux-outputs improves OOD performance over the baseline, and In-N-Out (with or without repeated ST) generally improves both in- and out-of-distribution performance over all other models.

\paragraph{CelebA.}
In CelebA, using auxiliary information either as aux-inputs or outputs improves both ID (2--4\%) and OOD accuracy (5\%).
We hypothesize this is because the auxiliary information is quite robust. Figure~\ref{fig:celeba-correlation} shows that there is a significant correlation ($r=0.72$) between ID and OOD accuracy for 100 different sets of aux-inputs, supporting results on standard datasets~\citep{recht2019doimagenet,xie2020selftraining,santurkar2020breeds}.
In-N-Out achieves the best OOD performance and comparable ID performance even though there is no tradeoff between ID and OOD accuracy.

\paragraph{Remote sensing.}
In the remote sensing datasets, aux-inputs can induce a tradeoff where increasing ID accuracy hurts OOD performance.
In cropland prediction, even with a small geographic shift (US Midwest), the baseline model has a significant drop from ID to OOD accuracy (4\%).
The aux-inputs model improves ID accuracy almost 1\% above the baseline but OOD accuracy drops 6\%.
In land cover prediction, using climate information as aux-inputs decreases OOD accuracy by over 4\% compared to the baseline. The aux-outputs model improves OOD, but decreases ID accuracy by 3\% over the baseline.

\paragraph{Improving in-distribution accuracy over aux-outputs.} One of the main goals of the self-training step in In-N-Out is to improve the in-distribution performance of the aux-outputs model. We compare to oracle models that use a large amount of in-distribution labeled data to compare the gains from In-N-Out. In Landcover, the oracle model which uses 160k labeled ID examples gets 80.5\% accuracy. In-N-Out uses 16k labeled examples and 150k unlabeled ID examples (with 50k unlabeled OOD examples) and improves the ID accuracy of aux-output from 72.5\% to 77.4\%, closing most (62\%) of the gap. In Cropland, the oracle model achieves 95.6\% accuracy. Here, In-N-Out closes 80\% of the gap between aux-outputs and the oracle, improving ID accuracy from 95.1\% to 95.5\%.

\paragraph{Ablations with only pre-training or self-training.}
We analyze the individual contributions of self-training and pre-training in In-N-Out.
On both cropland and land cover prediction, In-N-Out outperforms standard self-training on pseudolabels from the aux-inputs model (In-N-Out without pre-training), especially on OOD performance, where In-N-Out improves by about 1\% and 2\% respectively.
Similarly, In-N-Out improves upon pre-training (aux-outputs model) both ID and OOD for both datasets.

\subsection{Choice of auxiliary inputs matters}
\label{sec:input-choice}

\begin{figure}[tbp]
\begin{minipage}{.49\textwidth}
  \centering
  \subfloat{\includegraphics[width=0.5\textwidth]{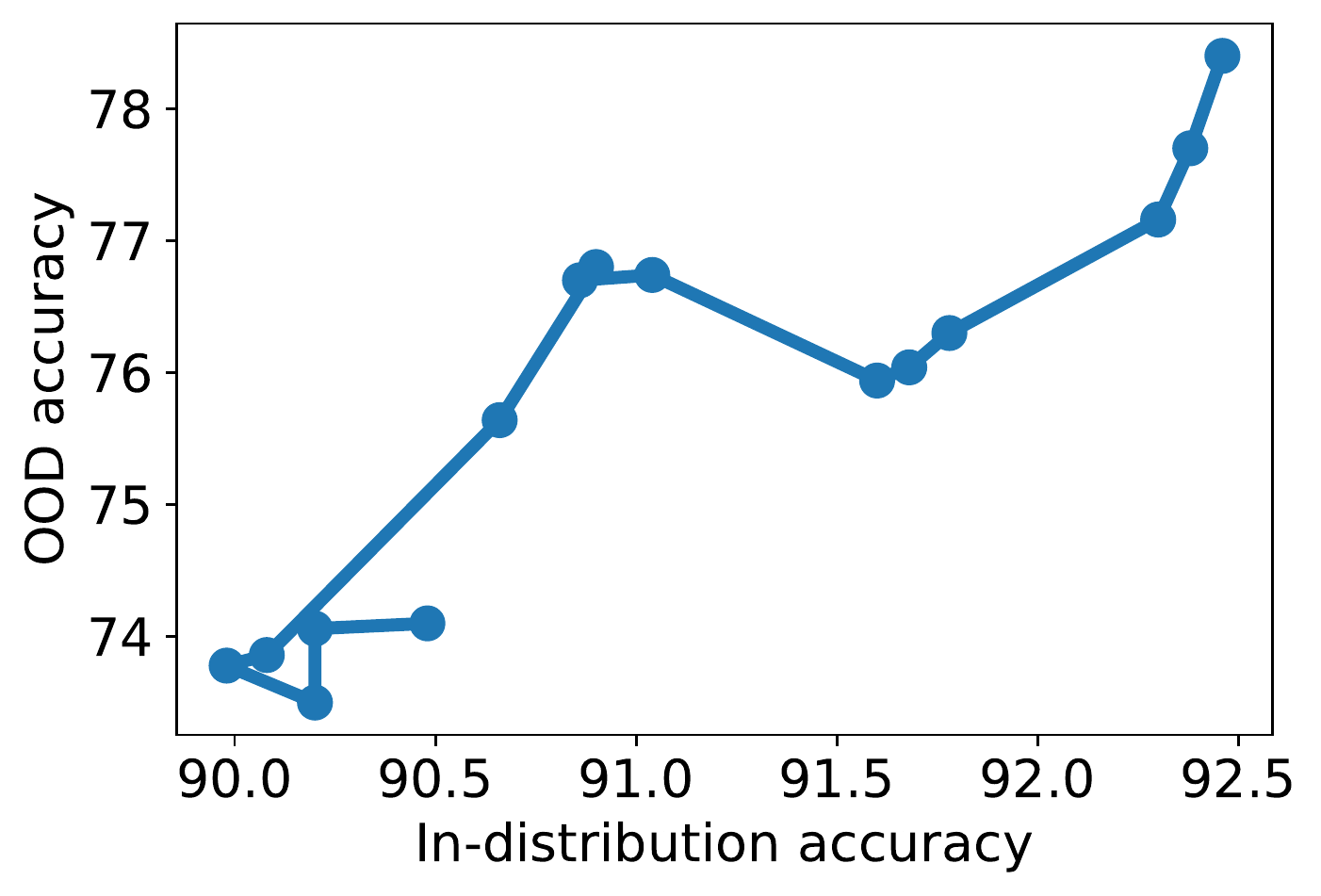}}
  \caption{In-distribution vs. OOD accuracy on CelebA when sequentially adding a random set of 15 auxiliary inputs one-by-one. Even if adding all 15 auxiliary inputs improves both in-distribution and OOD accuracy, some intermediate in-distribution gains can hurt OOD.}\label{fig:celeba-chain}
\end{minipage}
\hfill
\begin{minipage}{.49\textwidth}
  \centering
\scalebox{0.71}{
\begin{tabular} {l r r r}
\toprule
 & ID Test Acc & OOD Test Acc \\
\midrule
Only in-distribution & 69.73 $\pm$ 0.51 & 57.73 $\pm$ 1.58\\
Only OOD & 69.92 $\pm$ 0.41 & \textbf{59.28} $\pm$ 1.01\\
Both & 70.07 $\pm$ 0.46  & \textbf{59.84} $\pm$ 0.98\\
\bottomrule
\end{tabular}
}
\captionof{table}{Ablation study on the use of in-distribution vs. OOD unlabeled data in pre-training models on Landcover, where unlabeled sample size is standardized (much smaller than Table~\ref{table:main}). Using OOD unlabeled examples are important for gains in OOD accuracy (\%). Results are shown with 90\% error intervals over 5 trials.
}
\label{table:unlabeled-in-out}
\end{minipage}
\end{figure}

We find that the choice of auxiliary inputs affects the tradeoff between ID and OOD performance significantly, and thus is important to consider for problems with distribution shift.
While Figure~\ref{fig:celeba-correlation} shows that auxiliary inputs tend to simultaneously improve ID and OOD accuracy in CelebA, our theory suggests that in the worst case, there should be auxiliary inputs that worsen OOD accuracy.
Indeed, Figure~\ref{fig:celeba-chain} shows that when taking a random set of 15 auxiliary inputs and adding them sequentially as auxiliary inputs, there are instances where an extra auxiliary input improves in-distribution but hurts OOD accuracy even if adding all 15 auxiliary inputs improves both ID and OOD accuracy.
In cropland prediction, we compare using location coordinates and vegetation data as auxiliary inputs with only using vegetation data. The model with locations achieves the best ID performance, improving almost 1\% in-distribution over the baseline with only RGB. Without locations (only vegetation data), the ID accuracy is similar to the baseline but the OOD accuracy improves by 1.5\%. In this problem, location coordinates help with in-distribution interpolation, but the model fails to extrapolate to new locations.

\subsection{OOD unlabeled data is important for pre-training}
\label{sec:unlabeled-in-out}

We compare the role of in-distribution vs. OOD unlabeled data in pre-training.
Table~\ref{table:unlabeled-in-out} shows the results of using only in-distribution vs. only OOD vs. a balanced mix of unlabeled examples for pre-training on the Landcover dataset, where unlabeled sample size is standardized across the models (by reducing to the size of the smallest set, resulting in 4x less unlabeled data).
Using only in-distribution unlabeled examples does not improve OOD accuracy, while having only OOD unlabeled examples does well both in-distribution and OOD since it also has access to the labeled in-distribution data.
For the same experiment in cropland prediction, the differences were not statistically significant, perhaps due to the smaller geographic shift (across states in cropland vs. continents in landcover).

\section{Related work}
\label{sec:related-work}

\paragraph{Multi-task learning and weak supervision.}
\citet{caruana2003outputs} proposed using noisy features (aux-outputs) as a multi-task output, but do not theoretically analyze this approach.
\citet{wu2020multitask} also study multi-task linear regression.
However, their auxiliary tasks must have true parameters that are closely aligned (small cosine distance) to the target task. Similarly, weak supervision works assume access to weak labels correlated with the true label~\citep{ratner2016data,ratner2017snorkel}. In our paper, we make no assumptions about the alignment of the auxiliary and target tasks beyond a shared latent variable while also considering distribution shifts.

\paragraph{Transfer learning, pre-training, and self-supervision.}
We support empirical works that show the success of transfer learning and pre-training in vision and NLP~\citep{krizhevsky2012imagenet,simonyan2015verydeep,devlin2019bert}.
Theoretically, \citet{du2020fewshot,tripuraneni2020multitask} study pre-training in a similar linear regression setup.
They show in-distribution generalization bound improvements, but do not consider OOD robustness or combining with auxiliary inputs.
\citet{hendrycks2019selfsupervised} shows empirically that self-supervision can improve robustness to synthetic corruptions. We support these results by showing theoretical and empirical robustness benefits for pre-training on auxiliary information, which can be derived from the original input as in self-supervision.

\paragraph{Self-training for robustness.}
\citet{raghunathan2020understanding} analyze robust self-training (RST)~\citep{carmon2019unlabeled,najafi2019robustness,uesato2019are}, which improves the tradeoff between standard and adversarially robust accuracy, in min-norm linear regression. \citet{khani2021removing} show how to use RST to make a model robust against a predefined spurious feature without losing accuracy.
While related, we work in multi-task linear regression, study pre-training, and prove robustness to \emph{arbitrary} covariate shifts.
\citet{kumar2020gradual} show that repeated self-training on gradually shifting unlabeled data can enable adaptation over time. In-N-Out is complementary and may provide better pseudolabels in each step of this method.
\citet{chen2020selftraining} show that self-training can remove spurious features for Gaussian input features in linear models, whereas our results hold for general input distributions (with density).
\citet{zoph2020rethinking} show that self-training and pre-training combine for in-distribution gains. We provide theory to support this and also show benefits for OOD robustness.

\paragraph{Domain adaptation.}
Domain adaptation works account for covariate shift by using unlabeled data from a target domain to adapt the model~\citep{blitzer2007adaptation,daume07easyadapt,shu2018dirtt,hoffman2018cycada,ganin2016domain}. Often, modern domain adaptation methods~\citep{shu2018dirtt,hoffman2018cycada} have a self-training or entropy minimization component that benefits from having a better model in the target domain to begin with. Similarly, domain adversarial methods~\citep{ganin2016domain} rely on the inductive bias of the source-only model to correctly align the source and target distributions. In-N-Out may provide a better starting point for these domain adaptation methods.

\section{Discussion}
\label{sec:discussion}

\paragraph{Using spurious features for robustness.}
Counterintuitively, In-N-Out uses potentially spurious features (the auxiliary information, which helps in-distribution but hurts OOD accuracy) to improve OOD robustness.
This is in contrast to works on removing spurious features from the model~\citep{arjovsky2019invariant,ilyas2019adversarial,chen2020selftraining}.
In-N-Out promotes utilizing all available information by leveraging spurious features as useful in-distribution prediction signals rather than throwing them away.

\paragraph{General robustness with unlabeled data.}
In-N-Out is an instantiation of a widely applicable paradigm for robustness: collect unlabeled data in all parts of the input space and learn better representations from the unlabeled data before training on labeled data. This paradigm has driven large progress in few-shot generalization in vision~\citep{hendrycks2019pretraining,hendrycks2019selfsupervised} and NLP~\citep{devlin2019bert,brown2020gpt3}. In-N-Out enriches this paradigm by proposing that some features of the collected data can be used as input and output simultaneously, which results in robustness to arbitrary distribution shifts.

\paragraph{Leveraging metadata and unused features in applications.}
Many applications have inputs indexed by metadata such as location coordinates or timestamps~\citep{christie2018fmow,yeh2020poverty,ni2019justifying}.
 We can use such metadata to join (in a database sense) other auxilary data sources on this metadata for use in In-N-Out. This auxiliary information may often be overlooked or discarded, but In-N-Out provides a way to incorporate them to improve both in- and out-of-distribution accuracy.

\paragraph{Division between input features and auxiliary information.}
While a standard division between inputs and auxiliary information may exist in some domains, In-N-Out applies for any division of the input.
An important further question is how to automatically choose this division under distribution shifts.

\section{Conclusion}
\label{sec:conclusion}
We show that while auxiliary information as inputs improve in-distribution and OOD on standard curated datasets, they can hurt OOD in real-world datasets.
In contrast, we show that using auxiliary information as outputs by pretraining improves OOD performance. In-N-Out combines the strengths of auxiliary inputs and outputs for further improvements both in- and out-of-distribution.

\section{Acknowledgements}
We thank Sherrie Wang and Andreas Schlueter for their help in procuring remote sensing data, Daniel Levy for his insight in simplifying the proof of Theorem~\ref{thm:auxOutputHelpsEverywhere}, Albert Gu for a key insight in proving Lemma~\ref{lem:min-sing-lemma} using tools from random matrix theory, as well as Shyamal Buch, Pang Wei Koh, Shiori Sagawa, and anonymous reviewers for their valuable help and comments. This work was supported by an Open Philanthropy Project Award, an NSF Frontier Award as part of the Center for Trustworthy Machine Learning (CTML). SMX was supported by an NDSEG Fellowship. AK was supported by a Stanford Graduate Fellowship. TM was partially supported by the Google Faculty Award, JD.com, Stanford Data Science Initiative, and the Stanford Artificial Intelligence Laboratory.

\section{Reproducibility}
All code, data, and experiments are on CodaLab at \href{https://worksheets.codalab.org/worksheets/0x2613c72d4f3f4fbb94e0a32c17ce5fb0}{\color{blue}this link.}

\bibliographystyle{plainnat}
\bibliography{main}
\newpage
\appendix
\section{Proof for Sections~\ref{sec:input-output-theory} and~\ref{sec:in-n-out}}
\label{appendix:proofs}

Our theoretical setting assumes all the model families are linear.
We begin by specializing the setup in Section~\ref{sec:setup} and defining all the necessary matrices.
A word on notation: if unspecified, expectations are taken over all random variables.

\textbf{Data matrices}:
We have finite labeled data in-distribution: $n \geq d + m$ input examples $X \in \mathbb{R}^{n \times d}$, where each row $X_i \sim P_x$ is an example sampled independently.
We have an \emph{unobserved} latent matrix: $U \in \R^{n \times m}$ where each row $U_i \sim P_u$ is sampled independently from other rows and from $X$.
$U$ is unobserved and not directly used by any of the models, but we will reference $U$ in our analysis.
As stated in the main paper, we assume that $\E_{u \sim P_u}[u] = 0$.
We have labels $Y \in \mathbb{R}^{n}$ and auxiliary data $Z \in \mathbb{R}^{n \times T}$, where each row $Y_i, Z_i$ is sampled jointly given input example $X_i, U_i$, that is: $Y_i, Z_i \sim P_{y, z \mid X_i, U_i}$.
In our linear setting, we have $Z = X \Bstar^\top \Astar^\top + U \Cstar^\top$ and $Y = X \Bstar^\top \tw + U \tu + \epsilon$, where $\epsilon \in \R^n$ with each entry $\epsilon_i \sim P_{\epsilon}$ sampled independently from a mean $0$, variance $\sigma^2$ distribution $P_{\epsilon}$.

\textbf{Reminder on shapes}:
As a reminder, $\Bstar \in \R^{k \times d}$ maps the input $x \in \R^d$ to a low dimensional representation $w \in \R^k$ via $w = \Bstar x$.
$\Astar \in \R^{T \times k}, \Cstar \in \R^{T \times m}$ generate auxiliary $z \in \R^T$ via: $z = \Astar w + \Cstar u$.
Finally, $y \in \R$ is given by: $y = \dotp{\tw}{w} + \dotp{\tu}{u} + \epsilon$, where $\tw \in \R^k, \tu \in \R^m$.
Letting $\tx = \Bstar^\top \tw$, we equivalently have $y = \dotp{\tw}{x} + \dotp{\tu}{u} + \epsilon$ in terms of $x, u$.

\begin{figure}[tbp]
  \centering
  \includegraphics[width=0.4\textwidth]{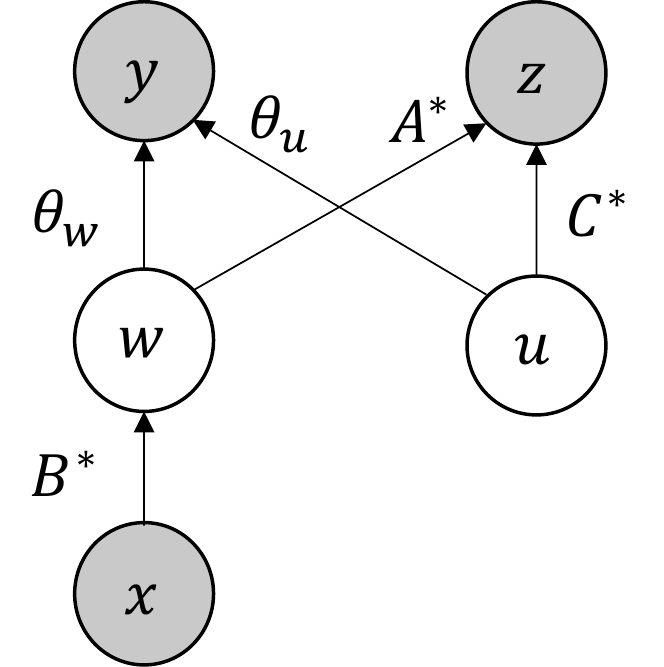}
  \caption{Graphical model for our theoretical setting, where auxiliary information $z$ is related to targets $y$ through the latent variable $w$ and latent noise $u$.}\label{fig:graphical-model-edges}
\end{figure}

\subsection{Models and evaluation}

\textbf{Baseline}: ordinary least squares estimator that uses $x$ only, so $\txbaseline = \argmin_{\theta'} \| Y - X \theta' \|_2$.
Given a test example $x, z$, the baseline method predicts $\fbs(x, z) = \dotp{\txbaseline}{x}$, ignoring $z$.
In closed form, $\txbaseline = (X^\top X)^\minv X^\top Y$.

\textbf{Aux-inputs}: least squares estimator using $x$ and auxiliary $z$ as input: $\txinput, \tzinput = \argmin_{\tx', \tz'} \| Y - (X \tx' + Z \tz')\|_2$.
The aux-inputs method predicts $\dotp{\txinput}{x} + \dotp{\tzinput}{z}$ for a test example $x, z$.
In closed form, letting $\Xz = [X; Z]$, where we append the columns so that $\Xz \in \mathbb{R}^{n \times (d + T)}$, $[\txinput, \tzinput] = ({\Xz}^\top {\Xz})^\minv {\Xz}^\top Y$.

\textbf{Aux-outputs}: pretrains on predicting $z$ from $x$ on unlabeled data to learn a mapping from $x$ to $w$, then learns a regression model on top of this latent embedding $w$. In the \emph{pre-training} step: use unlabeled data to learn the feature-space embedding $\hatB$:
\begin{equation}
\hatA, \hatB = \argmin_{A, B} \E_{x \sim P_x}[\| ABx - z \|_2^2] \quad A \in \mathbb{R}^{T \times k}, B \in \mathbb{R}^{k \times d}.
\end{equation}
The \emph{transfer} step solves a lower dimensional regression problem from $w$ to $y$: $\twoutput = \argmin_{\tw'} \| Y - X\hatB^{\top} \tw' \|_2$.
Given a test example $x$, the aux-outputs model predicts $\dotp{\twoutput}{\hatB x}$.

\textbf{In-N-Out}:
First learn an output model $\hatA, \hatB$, and let $W = X \hatB^{\top}$ be the feature matrix.
Next, train an input model on the feature space $w$.
\begin{equation}
\hat{\gammaw}, \hat{\gammaz} = \argmin_{\hat{\gammaw}, \hat{\gammaz}} \|Y - (W \hat{\gammaw} + Z \hat{\gammaz})\|_2.
\end{equation}
Note that this is slightly different from our experiments where we trained the aux-inputs model directly on the inputs $x$.
We now use the input model to pseudolabel our in-domain unlabeled examples, and self-train a model \emph{without z} on these pseudolabels.
Given each point $w, z$, we produce a pseudolabel $\dotp{\hat{\gammaw}}{w} + \dotp{\hat{\gammaz}}{z}$.
We now learn a least squares estimator from $w$ to the pseudolabels which gives us the In-N-Out estimator $\twinnout$:
\begin{equation}
\label{eqn:innoutLinearDfn}
\twinnout = \argmin_{\twinnout} \E_{w, z \sim \Pid} \Big[ \big(\dotp{\twinnout}{w} - (\dotp{\hat{\gammaw}}{w} + \dotp{\hat{\gammaz}}{z}) \big)^2 \Big]
\end{equation}
Given a test example $x$, In-N-Out predicts $\dotp{\twinnout}{\hatB x}$.

\subsection{Auxiliary inputs help in-distribution}

The proof of Proposition~\ref{prop:auxInputInDomain} is fairly standard.
We first give a brief sketch, specify the additional regularity conditions, and then give the proof.
We lower bound the risk of the baseline by $\sigmau^2 + \sigma^2$ since this is the Bayes-opt risk of using only $x$ but not $z$ to predict $y$.
We upper bound the risk of the aux-inputs model which uses $x, z$ to predict $y$, which is the same as upper bounding the risk in random design linear regression.
For this upper bound we use Theorem 1 in~\citet{hsu2012random} (note that there are multiple versions of this paper, and we specifically use the Arxiv version available at \url{https://arxiv.org/abs/1106.2363}).
As such, we inherit their regularity conditions.
In particular, we assume:
\begin{enumerate}
\item $x, u$ are upper bounded almost surely. This is a technical condition, and can be replaced with sub-Gaussian tail assumptions~\citep{hsu2012random}.
\item The noise $\epsilon$ is sub-Gaussian with variance parameter $\sigma^2$.
\item The latent dimension $m$ and auxiliary dimension $T$ are equal so that the inputs to the aux-inputs model have invertible covariance matrix.\footnote{$x$ and $u$ are independent, with invertible covariance matrices, and $z = \Astar \Bstar x + \Cstar u$ where $\Cstar$ is full rank, so by block Gaussian elimination we can see that $[x, z]$ has invertible covariance matrix as well.}
\end{enumerate}

\newtheorem*{auxInputInDomainProposition}{Restatement of Proposition~\ref{prop:auxInputInDomain}}

\begin{auxInputInDomainProposition}
\auxInputInDomainText{}
\end{auxInputInDomainProposition}

\begin{proof}
\textbf{Lower bound risk of baseline}: First, we lower bound the expected risk of the baseline by $\sigmau^2 + \sigma^2$.
Intuitively, this is the irreducible error---no linear classifier using only $x$ can get better risk than $\sigmau^2 + \sigma^2$ because of intrinsic noise in the output $y$.
Let $\tx = \Bstar^{\top} \tw$ be the optimal baseline parameters.
We have
\begin{align}
\Rid(\fbs) &= \E_{x, y, z \sim \Pid}[(y - \dotp{\txbaseline}{x})^2] \\
&= \E_{x, u, \epsilon \sim \Pid}[((\dotp{\tx}{x} + \dotp{\tu}{u} + \epsilon) - \dotp{\txbaseline}{x})^2] \\
&= \E_{x \sim \Pid}[(\dotp{\tx}{x} - \dotp{\txbaseline}{x})^2] + \E_{u \sim \Pid}[\dotp{\tu}{u}^2] + \E_{\epsilon \sim \Pid}[\epsilon^2] \label{eqn:decomp_baseline} \\
&\geq \E_{u \sim \Pid}[\dotp{\tu}{u}^2] + \E_{\epsilon \sim \Pid}[\epsilon^2]. \\
&= \sigmau^2 + \sigma^2.
\end{align}
To get Equation~\ref{eqn:decomp_baseline}, we expand the square, use linearity of expectation, and use the fact that $x, u, \epsilon$ are independent where $u, \epsilon$ are mean 0.

\textbf{Upper bound risk of aux-inputs}: On the other hand, we will show that if $n$ is sufficiently large, the expected risk of the input model is less than $\sigmau^2 + \sigma^2$.

First we show that we can write $y = \dotp{\tx'}{x} + \dotp{\tz'}{z} + \epsilon$ for some $\tx', \tz'$, that is $y$ is a well-specified linear function of $x$ and $z$ plus some noise. Intuitively this is because $y$ is a linear function of $x, u$ and since $\Cstar$ is invertible we can extract $u$ from $x, z$. Formally, we assumed the true model is linear, that is, $y = \dotp{\tx}{x} + \dotp{\tu}{u} + \epsilon$. Since we have $z = \Astar \Bstar x + \Cstar u$ where $\Cstar$ is invertible, we can write $u = \Cstar^{-1} (z - \Astar \Bstar x)$. This gives us

\begin{align}
y &= \dotp{\tx}{x} + \dotp{\tu}{u} + \epsilon \\
&= \dotp{\tx}{x} + \dotp{\tu}{\Cstar^{-1} (z - \Astar \Bstar x)} + \epsilon \\
&=  \dotp{(\tx - \Bstar^\top \Astar^\top (\Cstar^{\top})^{-1} \tu)}{x} + \dotp{(\Cstar^{\top})^{-1}\tu}{z} + \epsilon.
\end{align}

So setting $\tx' = \tx - \Bstar^\top \Astar^\top (\Cstar^{\top})^{-1} \tu$ and $\tz' = {\Cstar^\top}^{-1}\tu$, we get $y = \dotp{\tx'}{x} + \dotp{\tz'}{z} + \epsilon$.

As before, we note that the total mean squared error can be decomposed into the Bayes-opt error plus the excess error:
\begin{align}
\Rid(\fin) &= \E_{x, y, z \sim \Pid}[(y - \dotp{\txinput}{x} - \dotp{\tzinput}{z})^2] \\
&= \E_{x, z, \epsilon \sim \Pid}[((\dotp{\tx'}{x} + \dotp{\tz'}{z} + \epsilon) - \dotp{\txinput}{x} - \dotp{\tzinput}{z})^2] \\
&= \E_{x, z \sim \Pid}[(\dotp{\tx'}{x} + \dotp{\tz'}{z} - \dotp{\txinput}{x} - \dotp{\tzinput}{z})^2] + \E_{\epsilon \sim \Pid}[\epsilon^2] \label{eqn:decomp_input} \\
&= \E_{x, z \sim \Pid}[(\dotp{\tx'}{x} + \dotp{\tz'}{z} - \dotp{\txinput}{x} - \dotp{\tzinput}{z})^2] + \sigma^2. \label{eqn:final_decomp_input}
\end{align}
To get Equation~\ref{eqn:decomp_input}, we expand the square, use linearity of expectation, and use the fact that $x, z, \epsilon$ are independent with $\E[\epsilon] = 0$.
So it suffices to bound the excess error, defined as:

\begin{equation}
\EE := \E_{x, z \sim \Pid}[(\dotp{\tx'}{x} + \dotp{\tz'}{z} - \dotp{\txinput}{x} - \dotp{\tzinput}{z})^2]
\end{equation}

To bound the excess error, we use Theorem 1 in~\citet{hsu2012random} where the inputs/covariates are $[x, z]$.
\ak{honestly we only need consistency of the estimator for this (weak) result, but using Theorem 1 we can get explicit dependencies---I'm lazy though}
$\E[[x, z] [x, z]^{\top}]$ is invertible because $m = T$, $\Cstar$ is full rank, and $P_x, P_u$ have density everywhere with density upper bounded so the variance in any direction is positive, and so the population covariance matrix is positive definite.
This means $\E[[x, z] [x, z]^{\top}]$ has min singular value lower bounded, and we also have that $x$, $z$ are bounded random variables.
Therefore, Condition 1 is satisfied for some finite $\rho_0$. 
Condition 2 is satisfied since the noise $\epsilon$ is sub-Gaussian with mean $0$ and variance parameter $\sigma^2$.
Condition 3 is satisfied with $b_0 = 0$, since we are working in the setting of well-specified linear regression.

To apply Theorem 1~\citep{hsu2012random}, we first choose $t = \log{\frac{3}{\delta}}$ so that $1 - 3e^{-t} \geq 1 - \delta$, and so the statement of the Theorem holds with probability at least $1 - \delta$. Since our true model is linear (or as Remark 9 says that ``the linear model is correct''), $\text{approx}(x) = 0$.

So as per remark 9~\citep{hsu2012random} Equation 11, for some constant $c'$, we have an upper bound on the excess error $\EE$ with probability at least $1 - \delta$:
\begin{equation}
\EE \leq \frac{\sigma^2(d + 2 \sqrt{dt} + 2t)}{n} + o(1/n).
\end{equation}
Note that the notation in~\citet{hsu2012random} is different. The learned estimator in ordinary least squares regression is denoted by $\hat{\beta_0}$, the ground truth parameters by $\beta$, and the excess error is denoted by $\| \hat{\beta_0} - \beta \|_{\Sigma}$. See section 2.1, 2.2 of~\citet{hsu2012random} for more details.

Since $t$ is fixed, there exists some constant $c$ (dependent on $\delta$) such that for large enough $N_1$ if $n \geq N_1$:
\begin{equation}
\label{eqn:excess-error-bound-aux-inputs}
\EE \leq \sigma^2(cd/n).
\end{equation}
Note that this is precisely Remark 10~\citep{hsu2012random}. Remark 10 says that $\| \hat{\beta_0} - \bar{\beta_0} \|_{\Sigma}$ is within constant factors of $\sigma^2 d / n$ for large enough $n$. This is the variance term, but the bias term is $0$ since the linear model is well-specified so $\text{approx}(x) = 0$. As in Propostion 2~\citep{hsu2012random} the total excess error is bounded by 2 times the sum of the bias and variance term, which gives us the same result.

Putting this (Equation~\ref{eqn:excess-error-bound-aux-inputs}) back into Equation~\ref{eqn:final_decomp_input}, we get that with probability at least $1 - \delta$:
\begin{equation}
\Rid(\fin) \leq \sigma^2(1 + cd/n).
\end{equation}

Since $\sigmau^2 > 0$, we have $\sigma^2 < \sigmau^2 + \sigma^2$.
Then for some $N$ and for all $n \geq N$, we have
\begin{equation}
\Rid(\fin) < \sigmau^2 + \sigma^2 \leq \Rid(\fbs).
\end{equation}
In particular, we can choose $N = \max(N_1, c \frac{\sigma^2}{\sigmau^2} d + 1)$, which completes the proof.
\end{proof}

\subsection{Auxiliary inputs can hurt out-of-distribution}

\newtheorem*{auxInputOODExample}{Restatement of Example~\ref{ex:input-model-bad}}

\begin{auxInputOODExample}
\auxInputOODExampleText{}
\end{auxInputOODExample}

\begin{proof}
We will have $x \in \mathbb{R}$ (so $d = 1$), $w = x$, and $u, z \in \mathbb{R}^2$.
We set $z_1 = u_1 + w$ and $z_2 = u_2$, in other words we choose $\Astar = [1, 0]$ and $\Cstar = I_2$ is the identity matrix.
We set $y = x + u_1 + \epsilon$, with $\epsilon \sim N(0, \sigma^2)$, so $y$ is a function of $x$ and $u_1$ but not $u_2$.
In other words we choose $\tw = 1$ and $\tu = (1, 0)$.
$P_x$ will be $\mbox{Uniform}[-1, 1]$, and $P_u$ will be uniform in the unit ball in $\mathbb{R}^2$.

Let $\Xz = [X; Z]$, which denotes appending $X$ and $Z$ by columns so $\Xz \in \mathbb{R}^{n \times 3}$ with $n \geq 3$.
Since $P_x$ and $P_u$ have density, $\Xz$ has rank 3 almost surely.
This means that ${\Xz}^{\top} \Xz$ is invertible (and positive semi-definite) almost surely.
Since $P_x$ and $P_u$ are bounded, the maximum eigenvalue $\eigenmax'$ of ${\Xz}^{\top} \Xz$ is bounded above.
The minimum eigenvalue $\eigenmin$ of $({\Xz}^{\top} \Xz)^{-1}$ is precisely $1/\eigenmax'$ and is therefore positive and bounded below by some $c > 0$ almost surely.

We will define $P_x'$ and $P_u'$ soon.
For now, consider a new test example $x' \sim P_x', u' \sim P_u'$ with $z' = [x', 0] + u'$ and $y' = x' + u_1' + \epsilon'$ with $\epsilon' \sim N(0, \sigma^2)$ and $\E[u'] = 0$.
For the input model we have:
\begin{align}
\E[\Rood(\fin)]
&= \E[(y' - (\dotp{\txinput}{x'} + \dotp{\tzinput}{z'}))^2] \\
&= \sigma^2(1 + \E[(x', z')^{\top} ({\Xz}^{\top} \Xz)^{-1} (x', z')]) \\
&\geq \sigma^2(1  + \E[\eigenmin \| (x', z') \|_2^2]) \\
&\geq \sigma^2(1  + c \; \E[\| (x', z') \|_2^2]) \\
&\geq \sigma^2(1 + c \; \E[{z_2'}^2]) \\
&= \sigma^2(1 + c \; \E[{u_2'}^2])
\end{align}
Notice that this lower bound is a function of $\E[{u_2'}^2]$ which we will make very large. 

On the other hand, letting ${\sigmau'}^2 = \E_{u' \sim P_u'}[(\dotp{\tu}{u'})^2] = \E_{u' \sim P_u'}[{u_1'}^2]$, for the baseline model we have
\begin{align}
\E[\Rood(\fbs)] &= \E[(y' - \dotp{\txbaseline}{x'})^2] \\
&= \E[((\dotp{\tx}{x'} + \dotp{\tu}{u'} + \epsilon') - \dotp{\txbaseline}{x'})^2] \\
&= \E[(\dotp{\tu}{u'} + \epsilon')^2] + \E[((\dotp{\tx}{x'} - \dotp{\txbaseline}{x'})^2] \\
&= \E[(\dotp{\tu}{u'})^2] + \E[\epsilon'^2] + \E[(\dotp{\tx}{x'} - \dotp{\txbaseline}{x'})^2] \\
&= {\sigmau'}^2 + \sigma^2 + \E[(\dotp{\tx}{x'} - \dotp{\txbaseline}{x'})^2] \\
&= {\sigmau'}^2 + \sigma^2 + \E[x'^\top (\tx - \txbaseline) (\tx - \txbaseline)^\top x'] \label{eqn:txTxBaselineError} \\
&= {\sigmau'}^2 + \sigma^2 + (\sigma^2 + {\sigmau}^2)\E[{x'}^{\top} (X^{\top} X)^{-1} x']
\end{align}
where in Equation~\ref{eqn:txTxBaselineError}, we use the fact that $\tx - \txbaseline = (X^\top X)^{-1} X^\top (U \tu + \epsilon)$ to get the next line.
So the risk depends on $x'$ and $\E[{u_1'}^2]$ but not $\E[{u_2'}^2]$.

So we choose $P_x' = \mbox{Uniform}(-1, 1)$.
For $P_u'$, we sample the components independently, with $u_1' \sim \mbox{Uniform}(-1, 1)$, and $u_2' \sim \mbox{Uniform}(-R, R)$.
By choosing $R$ large enough, we can make the lower bound for the input model arbitrarily large without impacting the risk of the baseline model which gives us
\begin{equation}
\E[\Rood(\fin)] > \E[\Rood(\fbs)].
\end{equation}

\end{proof}

\subsection{Pre-training improves risk under arbitrary covariate shift}

\newtheorem*{auxOutputHelpsEverywhereTheorem}{Restatement of Theorem~\ref{thm:auxOutputHelpsEverywhere}}

\begin{auxOutputHelpsEverywhereTheorem}
\auxOutputHelpsEverywhereText{}
\end{auxOutputHelpsEverywhereTheorem}

First we show that pre-training (training a low-rank linear map from $x$ to $z$) recovers the unobserved features $w$.
We will then show that learning a regression map from $w$ to $y$ is better in all directions than learning a regression map from $x$ to $y$.

Our first lemma shows that we can recover the map from $x$ to $w$ up to identifiability (i.e., we will learn the rowspace of the true linear map from $x$ to $w$).

\begin{lemma}
    \label{thm:latent-rowspace}
For a pair $(x,z)$, let $z= \Astar\Bstar x + \xi$ where $\Astar\in \R^{T\times k}$ and $\Bstar\in \R^{k\times d}$ are the true parameters with $T,d\geq k$ and $\xi \in \R^{T}$ is mean-zero noise with bounded variance in each coordinate.
Assume that $\Astar,\Bstar$ are both rank $k$.
Suppose that $\E[x x^{\top}]$ is invertible.
Let $\hatA,\hatB$ be minimizers of the population risk $\E[\|\hatA\hatB x - z\|^2]$ of the multiple-output regression problem.
Then $\vspan\{\Bstar_1, \dots, \Bstar_k\} = \vspan\{\hatB_1,\dots,\hatB_k\}$ where $\Bstar_i, \hatB_i$ are the $i$-th rows of their respective matrices.
\end{lemma}
\begin{proof}
We first consider solving for the product of the weights $\hatA\hatB$.
Letting $C_i$ denote the $i$-th row of $C$, the population risk can be decomposed into the risks of the $T$ coordinates of the output:
\begin{align}
    \E[\|Cx - z\|^2]
    &= \sum_{i=0}^T \E[(C_i^\top x - z_i)^2] \\
    &= \sum_{i=0}^T \E[(C_i^\top x - (\Astar\Bstar)_i^\top x - \xi_i)^2] \\
    &= \sum_{i=0}^T \E[(C_i^\top x - (\Astar\Bstar)_i^\top x)^2] + \E[\xi_i^2]
\end{align}

Each term in the sum is the ordinary least squares regression loss, so a standard result is that since $\E[xx^T]$ is invertible, the unique minimizer is $C_i = (\Astar\Bstar)_i$.
One way to see this is to note that the loss is convex in $C$, and (by taking derivatives) if $\E[xx^T]$ is invertible the unique stationary point is $C = \Astar \Bstar$.
Therefore, we have that the product of the learned parameters and the true parameters are equal:
\begin{equation}
    \hatA\hatB = \Astar\Bstar
\end{equation}

By e.g., Sylvester's rank inequality, $\Astar \Bstar$ must be rank $k$, and so $\hatA \hatB$ is rank $k$ (since they are equal).
This means that $\hatA, \hatB$ are each rank $k$.
Now $\Astar \Bstar$ and $\hatA \hatB$ have the same rowspace because they are equal.
The rowspace of $\Astar \Bstar$ is a subspace of the rowspace of $\Bstar$, but both have rank $k$ so they are equal.
Similarly the rowspace of $\hatA \hatB$ and $\hatB$ are equal.
This implies that the rowspace of $\hatB$ and $\Bstar$ are equal, which is the desired result.
\end{proof}

Our next lemma shows that for any fixed training examples $X$ and \emph{arbitrary} test example $x'$, the aux-outputs model will have better expected risk than the baseline where the expectation is taken over the training labels $Y \mid X$.

\begin{lemma}
\label{lem:baselineVsOutputsExcess}
In the linear setting, fix data matrix $X$ and consider arbitrary test example $x'$.
Let $\thetastar = \Bstar^{\top} \tw$ be the optimal (ground truth) linear map from $x$ to $y$.
The expected excess risk of the aux-outputs model $\hatB^{\top} \twoutput$ is better than for the baseline $\txbaseline$, where the expectation is taken over the training targets $Y \sim P_{Y \mid X}$ ($Y$ shows up implicitly because the estimators $\twoutput$ and $\txbaseline$ depend on $Y$):
\begin{equation}
\E[(\dotp{\twoutput}{\hatB x'} - \dotp{\thetastar}{x'})^2] \leq \E[(\dotp{\txbaseline}{x'} - \dotp{\thetastar}{x'})^2]
\end{equation}
\end{lemma}

\begin{proof}
Let $\epsilon_{all} = Y - X \thetastar$ be the training noise.
From standard calculations, the instance-wise risk of $\txbaseline$ for any $x$ is
\begin{align}
    \E[(\txbaseline^\top x' - \thetastar^\top x')^2] &= \E[(((X^\top X)^\minv X^\top Y)^\top x' - \thetastar^\top x')^2]\\
                                                          &= \E[((\thetastar + (X^\top X)^\minv X^\top \epsilon_{all})^\top x' - \thetastar^\top x')^2]\\
                                                          &= \E[(((X^\top X)^\minv X^\top \epsilon_{all})^\top x')^2]\\
                                                          &= (\sigma^2 + \sigmau^2) {x'}^\top (X^\top X)^\minv  x'
\end{align}
By Lemma~\ref{thm:latent-rowspace}, $\hatB=QB$ for some full rank $Q$.
Thus, learning $\twoutput$ is a regression problem with independent mean-zero noise and we can apply the same calculations for the instance-wise risk of $\hatB^\top\twoutput$.
\begin{equation}
    \E[(\twoutput^{\top}\hatB x' - \thetastar^\top x')^2] = (\sigma^2 + \sigmau^2) {x'}^\top \hatB^\top (\hatB X^\top X\hatB^\top)^\minv\hatB x'.
\end{equation}
We show that the difference between the inner matrices is positive semi-definite, which implies the result. In particular, we show that
\begin{align}
	\label{eqn:output-thm-psd}
    (X^\top X)^\minv - \hatB^\top (\hatB X^\top X\hatB^\top)^\minv\hatB \succcurlyeq 0.
\end{align}

Since $X^\top X$ is a full rank PSD matrix, we can write $X^\top X = G G^\top$ for $G \in \mathbb{R}^{d \times d}$ where $G$ is full rank and therefore invertible.
Expressing Equation~\ref{eqn:output-thm-psd} in terms of $G$, we want to show:
\begin{align}
    (G G^\top)^\minv - \hatB^\top (\hatB G G^\top \hatB^\top)^\minv\hatB \succcurlyeq 0.
\end{align}
Left multiplying by $G^\top$ and right multiplying by $G$, which are both invertible, this is equivalent to showing:
\begin{align}
    M := I - (\hatB G)^\top (\hatB G G^\top \hatB^\top)^\minv (\hatB G) \succcurlyeq 0.
\end{align}
But we note that $M$ is symmetric, with $M = M^2 = MM^\top$, so $M$ is PSD.
This completes the proof.

\end{proof}

\begin{proof}[Proof of Theorem~\ref{thm:auxOutputHelpsEverywhere}]
Fix training examples $X$ and test example $x'$ but let the train labels $Y \sim P_{Y \mid X}$ and and test label $y' \sim P_{y' \mid x'}'$ be random.
In particular, let ${\sigmau'}^2 = \E[(\dotp{\tu}{u'})^2]$ where $u' \sim P_u'$, with $\E[u'] = 0$.
Then for the baseline OLS estimator, we have:
\begin{equation}
\E[(y' - \dotp{\txbaseline}{x'})^2] = {\sigmau'}^2 + \sigma^2 + \E[(\txbaseline^\top x' - \thetastar^\top x')^2]
\end{equation}
For the aux-outputs model, we have:
\begin{equation}
\E[(y' - \dotp{\twoutput}{\hatB x'})^2] = {\sigmau'}^2 + \sigma^2 + \E[(\dotp{\twoutput}{\hatB x'} - \dotp{\thetastar}{x'})^2]
\end{equation}
So applying Lemma~\ref{lem:baselineVsOutputsExcess}, we get that the risk for the aux-outputs model is better than for the baseline (the lemma showed it for the excess risk):
\begin{equation}
\E[(y' - \dotp{\twoutput}{\hatB x'})^2] \leq \E[(y' - \dotp{\txbaseline}{x'})^2]
\end{equation}
Since this is true for all $X$ and $x'$, it holds when we take the expectation over the training examples $X$ from $P_x$ and the test example $x'$ from $P_x'$ which gives us the desired result.
\end{proof}

\subsection{In-N-Out improves risk under arbitrary covariate shift}

\newtheorem*{selfTrainingHelpsEverywhereTheorem}{Restatement of Theorem~\ref{thm:self-train}}

\begin{selfTrainingHelpsEverywhereTheorem}
\selfTrainingHelpsEverywhereText{}
\end{selfTrainingHelpsEverywhereTheorem}

We first show a key lemma that lets us bound the min singular values of a random matrix, which will let us upper bound the risk of the In-N-Out estimator and lower bound the risk of the pre-training estimator.

\begin{definition}
As usual, the min singular value $\singularmin(W)$ of a rectangular matrix $W \in \R^{n \times k}$ where $n \geq k$ refers to the $k$-th largest singular value (the remaining $n-k$ singular values are all $0$), or in other words:
\begin{equation}
\singularmin(W) = \min_{\|\nu\|_2 = 1} \| W \nu \|_2.
\end{equation}
\end{definition}

\begin{lemma}
\label{lem:min-sing-lemma}
Let $P_w$ and $P_u$ be independent distributions on $\R^k$ and $\R^m$ respectively.
Suppose they are absolutely continuous with respect to the standard Lebesgue measure on $\R^k$ and $\R^m$ respectively (e.g., this is true if they have density everywhere with density upper bounded).
Let $W \in \R^{n \times k}$ where each row $W_i$ is sampled independently $W_i \sim P_w$.
Let $U \in \R^{n \times m}$ where each row $U_i$ is sampled independently $U_i \sim P_u$.
Suppose $n \geq k + m$.
For all $\delta$, there exists $c(\delta) > 0$ such that with probability at least $1 - \delta$, the minimum singular values $\singularmin$ are lower bounded by $c(\delta)$: $\singularmin(W) > c(\delta)$ and $\singularmin([W; U]) > c(\delta)$.
\end{lemma}

\begin{proof}
We note that the matrices $W$ and $U$ are rectangular, e.g., $W \in \mathbb{R}^{n \times k}$ where $n \geq k$.
We will prove the lemma for $W$ first, and the extension to $[W; U]$ will follow.

Note that removing the last $n-k$ rows of $W$ cannot increase its min singular value since that corresponds to projecting the vector $W \nu$ and projection never increases the Euclidean norm.
So without loss of generality, we suppose $W$ only consists of its first $k$ rows and so $W \in \mathbb{R}^{k \times k}$.

Now, consider any row $W_i$. 
We will use a volume argument to show that with probability at least $1 - \frac{\delta}{k}$, this row $W_i$ has a non-trivial component perpendicular to all the other rows.
Since all rows are independently and identically sampled, without loss of generality suppose $i = 1$.
Fix the other rows $W_2, \ldots, W_k$, since $W_1$ is independent of these other rows, the conditional distribution of $W_1$ is the same as the marginal of $W_1$.
The remaining rows $W_2, \ldots, W_k$ form a $k-1$ dimensional subspace $S$ in $\mathbb{R}^k$.
Letting $d(w, S)$ denote the Euclidean distance of a vector $w$ from the subspace $S$, define the event $S_{\lambda} = \{ W_1 : d(W_1, S) \leq \lambda \}$.
Since $P_w$ is absolutely continuous, $P(S_{\lambda}) \to 0$ as $\lambda \to 0$, so for some small $c(\delta) > 0$, $P(S_{c(\delta)}) < \frac{\delta}{k}$. So with probability at least $1 - \delta/k$, $d(W_1, S) > c(\delta)$.

By union bound, with probability at least $1 - \delta$, the distance from $W_i$ to the subspace spanned by all the other rows is greater than $c(\delta)$ for every row $W_i$, so we condition on this.
By representing each row vector as the sum of the component perpendicular to $S$ and a component in $S$, applying Pythagoras theorem and expanding we get
\begin{equation}
\min_{\|\nu\|_2 = 1} \|W \nu\| = \min_{\|\nu\|_2 = 1} \|\nu^{\top} W \| \geq c(\delta).
\end{equation}
Which completes the proof for $\singularmin(W)$.

For $[W; U]$, we note that $P_x$ and $P_u$ are independent, and the product measure is absolutely continuous.
Since each row of $[W; U]$ is identically and independently sampled just like with $W$, we can apply the exact same argument as above (though for a different constant $c(\delta)$, we take the min of these two as our $c(\delta)$ in the lemma statement).
\end{proof}

Recall that the In-N-Out estimator was obtained by fitting a model from $w, z$ to $y$, and then using that to produce pseudolabels on (infinite) unlabeled data, and then self-training a model from $w$ to $y$ on these pseudolabels.
For the linear setting, we defined the In-N-Out estimator $\twinnout$ in Equation~\ref{eqn:innoutLinearDfn}.
Our next lemma gives an alternate closed form of the In-N-Out estimator in terms of the representation matrix $W = X \hatB$ and the latent matrix $U$.

\begin{lemma}
\label{lem:innout-alt-estimator-closed-form}
In the linear setting, letting $W = X \hatB^{\top}$ we can write the In-N-Out estimator in closed form:
\begin{equation}
\label{eqn:innout-alt-estimator-closed-form}
\twinnout = [I_{k \times k}; 0_{k \times T}] \Big( \begin{pmatrix} W^{\top} \\ U^{\top} \end{pmatrix} (W; U) \Big)^{\minv} \begin{pmatrix} W^{\top} \\ U^{\top} \end{pmatrix} Y.
\end{equation}
\end{lemma}

\begin{proof}
\newcommand{\tualt}{\hat{\tu'}}
\newcommand{\twalt}{\twinnout'}
We recall the definition of the In-N-Out estimator, where we first train a classifier from $W, Z$ to $Y$:
\begin{equation}
\label{eqn:innout-original-estimator-minimizer}
\hat{\gammaw}, \hat{\gammaz} = \argmin_{\hat{\gammaw}, \hat{\gammaz}} \|Y - (W \hat{\gammaw} + Z \hat{\gammaz})\|_2.
\end{equation}
Denote the minimum value of Equation~\ref{eqn:innout-original-estimator-minimizer} by $p^*$.
Note that $\hat{\gammaw}, \hat{\gammaz}$ may not be unique, and we pick any solution to the $\argmin$ (although our proof will reveal that the resulting $\twinnout$ is in fact unique).
We then use this to produce pseudolabels and self-train, on infinite data, which gives us the In-N-Out estimator:
\begin{equation}
\label{eqn:innoutLinearDfn}
\twinnout = \argmin_{\twinnout} \E_{w, z \sim \Pid} \Big[ \big(\dotp{\twinnout}{w} - (\dotp{\hat{\gammaw}}{w} + \dotp{\hat{\gammaz}}{z}) \big)^2 \Big]
\end{equation}
Subtituting $z = \Astar w + \Cstar u$, we can write the loss that the In-N-Out estimator $\twinnout$ minimizes as:
\begin{equation}
\E_{w, z \sim \Pid} \Big[ \big(\dotp{\twinnout}{w} - (\dotp{(\hat{\gammaw} + {\Astar}^\top \hat\gammaz)}{w} + \dotp{({\Cstar}^\top \hat{\gammaz})}{u}) \big)^2 \Big]
\end{equation}
We group the terms slightly differently:
\begin{equation}
\E_{w, z \sim \Pid} \Big[ \big((\dotp{\twinnout}{w} - \dotp{(\hat{\gammaw} + {\Astar}^\top \hat\gammaz)}{w}) - \dotp{({\Cstar}^\top \hat{\gammaz})}{u} \big)^2 \Big]
\end{equation}
Expanding the square, using the fact that $\E[u] = 0$, and $u, w$ are independent, and ignoring terms with no dependency on $\twinnout$, this is equivalent to minimizing:
\begin{equation}
\E_{w \sim \Pid} \Big[ \big(\dotp{\twinnout}{w} - \dotp{(\hat{\gammaw} + {\Astar}^\top \hat\gammaz)}{w} \big)^2 \Big]
\end{equation}
This is minimized (indeed it is $0$) by setting:
\begin{equation}
\twinnout = \hat{\gammaw} + {\Astar}^\top \hat\gammaz
\end{equation}
The minimizer is unique because $w$ has invertible covariance matrix (since $x$ has invertible covariance matrix and $\Bstar$ is full rank), and so is in every direction with some probability.
We will now consider the following alternative estimator:
\begin{equation}
\label{eqn:innout-alt-estimator-minimizer}
\twalt, \tualt = \argmin_{\twalt, \tualt} \|Y - (W \twalt + U \tualt)\|_2.
\end{equation}
Denote the minimum value of Equation~\ref{eqn:innout-alt-estimator-minimizer} by $q^*$.
We claim that $\twalt = \twinnout$.

We will show that the In-N-Out estimator $\twinnout$ minimizes the alternative minimization problem in Equation~\ref{eqn:innout-alt-estimator-minimizer} by showing that $p^* = q^*$.
We will then show that the solution to Equation~\ref{eqn:innout-alt-estimator-minimizer} is unique, which implies that $\twalt = \twinnout$.

We note that $\Cstar^{\top} \in \mathbb{R}^{m \times T}$ where $T \geq m$ is full-rank, so there exists a right-inverse $C'$ with $\Cstar^{\top} C' = I_{m \times m}$.
Since $Z = W \Astar^{\top} + U \Cstar^{\top}$, this gives us $U = (Z - W \Astar^{\top}) C' = Z C' + W(-\Astar^{\top}C')$.

So this means that a solution to the alternative problem in Equation~\ref{eqn:innout-alt-estimator-minimizer} can be converted into a solution for the original in Equation~\ref{eqn:innout-original-estimator-minimizer} with the same function value:
\begin{align}
&\; \min_{\twalt, \tualt} \| Y - (W \twalt + U \tualt) \|_2 \\
=&\; \min_{\twalt, \tualt} \| Y - (W \twalt + (Z C' + W(-\Astar^{\top}C')) \tualt) \|_2 \\
=&\; \min_{\twalt, \tualt} \| Y - (W (\twalt -\Astar^{\top}C' \tualt) + Z (C' \tualt)) \|_2.
\end{align}
This implies that $p^* \leq q^*$.

We now show that a solution to the original problem in Equation~\ref{eqn:innout-original-estimator-minimizer} can be converted into a solution for the alternative in Equation~\ref{eqn:innout-alt-estimator-minimizer} with the same function value:
\begin{align}
&\; \min_{\hat{\gammaw}, \hat{\gammaz}} \| Y - (W \hat{\gammaw} + Z \hat{\gammaz}) \|_2 \\
=&\; \min_{\hat{\gammaw}, \hat{\gammaz}} \| Y - (W \hat{\gammaw} + (W \Astar^{\top} + U \Cstar^{\top}) \hat{\gammaz}) \|_2 \\
=&\; \min_{\hat{\gammaw}, \hat{\gammaz}} \| Y - (W (\hat{\gammaw} + \Astar^{\top}\hat{\gammaz}) + U (\Cstar^{\top} \hat{\gammaz})) \|_2.
\end{align}
This implies that $q^* \leq p^*$, and we showed before that $p^* \leq q^*$ so $p^* = q^*$.
But since $\hat{\gammaw}, \hat{\gammaz}$ minimizes the original minimizer in Equation~\ref{eqn:innout-original-estimator-minimizer}, $\hat{\gammaw} + \Astar^{\top}\hat{\gammaz}, \Cstar^{\top} \hat{\gammaz}$ minimize the alternative problem in Equation~\ref{eqn:innout-alt-estimator-minimizer}, where $\twinnout = \hat{\gammaw} + \Astar^{\top}\hat{\gammaz}$.

Since $[W; U]$ is full rank, the solution $\twalt, \tualt$ to the alternative estimator Equation~\ref{eqn:innout-alt-estimator-minimizer} is unique.
So this means that $\twalt = \twinnout$.

We have shown that $\twalt = \twinnout$---this completes the proof because solving Equation~\ref{eqn:innout-alt-estimator-minimizer} for $\twalt$ gives us the closed form in Equation~\ref{eqn:innout-alt-estimator-closed-form}.

\end{proof}

Next we show a technical lemma that says that if a random vector $u \in \R^n$ has bounded density everywhere, then for any $v$ with high probability the dot product $(\dotp{u}{v})^2$ cannot be too small relative to $\|v\|_2^2$.

\begin{lemma}
\label{lem:lower_bound_random_product}
Suppose a random vector $u \in \R^n$ has density everywhere, with bounded density.
For every $\delta$, there exists some $c(\delta)$ such that for all $v$, with probability at least $1 - \delta$ over $u$, $(\dotp{u}{v})^2 \geq c(\delta) \|v\|_2^2$.
\end{lemma}

\begin{proof}
First, we choose some $B_0$ such that $P(\|u\|_2 \geq B_0) \leq \delta/2$, such a $B_0$ exists for every probability measure.

Suppose that the density is upper bounded by $B_1$.
Let the area of the $n-1$ dimensional sphere with radius $B_0$ be $A_0$.
Consider any $n-1$ dimensional subspace $S$, and let $S_{\epsilon} = \{ u' : d(u', S) \leq \epsilon \}$ where $d(u', S)$ denotes the Euclidean distance from $u'$ to $S$.
We have $P(u \in S_{\epsilon}) \leq A_0 B_1 \epsilon + \delta/2$ for all $S_{\epsilon}$.
By choosing sufficiently small $\epsilon > 0$, we can ensure that $P(u \in S_{\epsilon}) \leq \delta$ for all $S$.

Now consider arbitrary $v$ and let $S(v)$ be the $n-1$-dimensional subspace perpendicular to $v$.
We have $P(u \in S(v)_{\epsilon}) \leq \delta$.
But this means that $(\dotp{u}{v})^2 \geq \epsilon^2 \|v\|_2^2$ with probability at least $1-\delta$, which completes the proof.
\end{proof}

By definition of our linear multi-task model, we recall that $y = \dotp{\tw}{w} + \dotp{\tu}{u} + \epsilon$, where $w = \Bstar x$.
We do not have access to $\Bstar$, but we assumed that $\Bstar$ is full rank.
We learned $\hatB$ which has the same rowspace as $\Bstar$ (Lemma ~\ref{thm:latent-rowspace}).
This means that for some $\tw'$, we have $y = \dotp{\tw'}{\hat{w}} + \dotp{\tu}{u} + \epsilon$ where $\hat{w} = \hatB x$.
To simplify notation and avoid using $\tw'$ and $\hat{w}$ everywhere, we suppose without loss of generality that $\hatB = \Bstar$ (but formally, we can just replace all the occurrences of $\tw$ by $\tw'$ and $w$ by $\hat{w}$).

Our next lemma lower bounds the test error of the pre-training model.

\begin{lemma}
\label{lem:self_train_outputs_lower_bound}
In the linear setting, for all problem settings $\probsetting$ with $\sigmau^2 > 0$, for all $\delta$, there exists some $a, b > 0$ such that with probability at least $1 - \delta$ over the training examples and test example $x' \sim P_x'$ the risk of the aux-outputs model is lower bounded:
\begin{equation}
\Routood - R^* > a - b\sigma^2.
\end{equation}
\end{lemma}

\begin{proof}
Recall that $\Routood = \E_{y' \sim P_{y'\mid x'}'}[l(\gy(\hout(x')), y')]$.
Let ${\sigmau'}^2 = \E_{u' \sim P_u'}[(\dotp{\tu}{u'})^2]$.
We have $R^* = \sigma^2 + {\sigmau'}^2$.
Let $W = X \Bstar^{\top}$ be the feature matrix, where $W \in \R^{n \times k}$.

Letting $w' = \Bstar x'$, for the aux-outputs model we have
\begin{align}
&\; \E_{y' \sim P_{y'\mid x'}'}[l(\gy(\hout(x')), y')] \\
=&\; \E_{y' \sim P_{y' \mid x'}}[ (y' - \dotp{\twoutput}{w'})^2 ] \\
=&\; (\sigma^2 + {\sigmau'}^2) + (\dotp{\tw}{w'} - \dotp{\twoutput}{w'})^2 \\
=&\; R^* + (\dotp{\tw}{w'} - \dotp{\twoutput}{w'})^2. \label{eqn:output_error_lemma_decomp}
\end{align}

Let $\epsilon = Y - (W \tw + U \tu)$ be the noise of $Y$ for the training examples, which is a random vector with $\epsilon \in \R^n$.
We can now write
\begin{equation}
\label{eqn:output_param_error_lemma}
(\dotp{\tw}{{w'}} - \dotp{\twoutput}{{w'}})^2 = ((\epsilon + U \tu)^{\top} W (W^{\top} W)^{\minv} {w'})^2.
\end{equation}
By assumption, $W^{\top} W$ is invertible (almost surely).
With probability at least $1 - \delta/10$ all entries of $W^{\top} W$ are upper bounded and we condition on this.
So $(W^{\top} W)^{\minv}$ has min singular value bounded below.
By Lemma~\ref{lem:min-sing-lemma}, $W$ has min singular value that is bounded below with probability at least $1 - \delta/10$. We condition on this being true.
So let $\nu = W (W^{\top} W)^{\minv} w'$, so for some $c_0 > 0$, we have $\| \nu \|_2 \geq c_0 \| w' \|_2$.

In terms of $\nu$, we can write Equation~\ref{eqn:output_param_error_lemma} as
\begin{align}
\label{eqn:output_param_error_lemma_nu}
(\dotp{\tw}{{w'}} - \dotp{\twoutput}{{w'}})^2 &= ((\epsilon + U \tu)^{\top} \nu)^2 \\
&= (\epsilon^{\top} \nu)^2 + ((U \tu)^{\top} \nu)^2 + 2(\epsilon^{\top} \nu) ((U \tu)^{\top} \nu) \\
&\geq ((U \tu)^{\top} \nu)^2 + 2(\epsilon^{\top} \nu) ((U \tu)^{\top} \nu) \\
&\geq ((U \tu)^{\top} \nu)^2 - 2|\epsilon^{\top} \nu| \|(U \tu)^{\top}\|_2 \|\nu\|_2). \label{eqn:first_lower_bound}
\end{align}

We can find $b_0$ such that with at least probability $1 - \delta/10$, $\|(U \tu)^{\top}\|_2 \leq b_0$, condition on this.
We note that $\epsilon^{\top} \nu$ has variance $\sigma^2 \|\nu\|_2$ so by Chebyshev for some $b_1$ with probability at least $1 - \delta/10$, $|\epsilon^{\top} \nu| \leq b_1 \sigma^2 \|\nu\|_2$, condition on this.
So we can now bound Equation~\ref{eqn:first_lower_bound} and get:
\begin{equation}
(\dotp{\tw}{{w'}} - \dotp{\twoutput}{{w'}})^2 \geq ((U \tu)^{\top} \nu)^2 - 2 b_0 b_1 \sigma^2 \| \nu \|_2^2
\end{equation}

Now we apply Lemma~\ref{lem:lower_bound_random_product}, where we use the fact that $\sigmau^2 > 0$.
So given $\delta/10$, there exists some $c_1$ such that for every $\nu$ with probability at least $1 - \delta/10$, $((U \tu)^{\top} \nu)^2 \geq c_1 \| \nu \|_2^2$, giving us
\begin{equation}
(\dotp{\tw}{{w'}} - \dotp{\twoutput}{{w'}})^2 \geq (c_1 - 2 b_0 b_1 \sigma^2) \| \nu \|_2^2.
\end{equation}

Since $w'$ has bounded density everywhere, it is non-atomic and we get that there is some $c_2 > 0$ such that with probability at least $1 - \delta/10$, $\| w' \|_2^2 \geq c_2^2$.
But then $\| \nu \|_2^2 \geq c_0 c_2$, which gives us for some $a, b$,
\begin{equation}
(\dotp{\tw}{{w'}} - \dotp{\twoutput}{{w'}})^2 \geq (c_1 - 2 b_0 b_1 \sigma^2) c_0 c_2 \geq a - b \sigma^2.
\end{equation}
\ak{So I initially got confused about how to combine these since the ordering of these things is critical. The point is that for each fixed $\xtest$ with high probability we get the above lower bound on ${\epsilon'}^{\top} \nu $, although the set of $\epsilon'$ satisfying the lower bound can be different. But that's OK, it's something like for each $x$, the good event happens 90\% of the time, so overall it happens 90\% of the time.}
Combining this with Equation~\ref{eqn:output_error_lemma_decomp}, we get with probability at least $1 - \delta$,
\begin{equation}
\E_{y' \sim P_{y'\mid x'}'}[l(\gy(\hout(x')), y')] - R^* > a - b\sigma^2,
\end{equation}
as desired.
\end{proof}

\begin{lemma}
\label{lem:self_train_innout_upper_bound}
In the linear setting, for all problem settings $\probsetting$, for all $\delta$, there exists some $c > 0$ such that with probability at least $1 - \delta$ over the training examples and test example $x' \sim P_x'$ the risk of the In-N-Out model is upper bounded:
\begin{equation}
\Rinnoutood - R^* < c \sigma^2.
\end{equation}
\end{lemma}

\begin{proof}
Recall that $\Rinnoutood = \E_{y' \sim P_{y'\mid x'}'}[l(\ginnout(\hout(x')), y')]$.
Let ${\sigmau'}^2 = \E_{u' \sim P_u'}[(\dotp{\tu}{u'})^2]$.
We have $R^* = \sigma^2 + {\sigmau'}^2$.
As before, let $W = X \Bstar^{\top}$ be the feature matrix, where $W \in \R^{n \times k}$.

Let $\Wu = [W; U]$ which denotes concatenating the matrices by column, so that $\Wu \in \mathbb{R}^{n \times (k + m)}$.
By Lemma~\ref{lem:min-sing-lemma}, ${\Wu}^{\top}\Wu$ has min singular value that is bounded below by $c_0$ with probability at least $1 - \delta/10$, we condition on this being true.
Now, as for the aux-outputs model, letting $w' = \Bstar x'$, we have
\begin{equation}
\E_{y' \sim P_{y' \mid w'}}[ (y' - \dotp{\twinnout}{w'})^2 ] = R^* + (\dotp{\tw}{w'} - \dotp{\twinnout}{w'})^2.
\end{equation}
For the second term on the RHS:
Let $R = [I_{k \times k}; 0_{k \times m}]$.
Let $\epsilon = Y - (W \tw + U \tu)$ be the noise of $Y$ for the training examples, which is a random vector with $\epsilon \in \R^n$.
From Lemma~\ref{lem:innout-alt-estimator-closed-form}, we can now write:
\begin{equation}
\label{eqn:w-err-rhs-term}
(\dotp{\tw}{w'} - \dotp{\twinnout}{w'})^2 = (w'^{\top} R ({\Wu}^{\top} {\Wu})^{\minv} {\Wu}^{\top} \epsilon )^2.
\end{equation}
$\|w'\|_2$ is bounded above by some constant $B_1$ with probability at least $1 - \delta/10$ which we condition on.
Now taking the expectation over $w'$ and $\epsilon$, using the fact that $R$ preserves the norm of a vector we can write
\begin{align}
& \; \E_{w', \epsilon} [ (w'^{\top} R ({\Wu}^{\top} {\Wu})^{\minv} {\Wu}^{\top} \epsilon )^2 ] \\
=& \; \sigma^2 \E_{w', \epsilon} [ (w'^{\top} R [{\Wu}^{\top} \Wu]^{\minv} R^T w') ] \\
\leq& \; \frac{\sigma^2}{c_0^2} \E_{w'}[\| w' \|_2^2] \\
\leq& \; \frac{B_1^2 \sigma^2}{c_0^2}.
\end{align}

Then, by Markov's inequality, with probability at least $1 - \delta / 10$ we can upper bound Equation~\ref{eqn:w-err-rhs-term} by $\frac{10B_1^2 \sigma^2}{\delta c_0^2}$.
In total, that gives us that for some $c$, with probability at least $1 - \delta$:
\begin{equation}
\label{eqn:innout-upperbound}
\E_{y' \sim P_{y'\mid x'}'}[l(\ginnout(\hout(x')), y')] - R^* < c \sigma^2.
\end{equation}
\end{proof}

The proof of Theorem~\ref{thm:self-train} simply combines Lemma~\ref{lem:self_train_outputs_lower_bound} and Lemma~\ref{lem:self_train_innout_upper_bound}.

\begin{proof}[Proof of Theorem~\ref{thm:self-train}]
For some $a, b, c$, with probability at least $1 - \delta$, we have for the aux-outputs model:
\begin{equation}
\Routood - R^* > a - b\sigma^2,
\end{equation}
and for the In-N-Out model:
\begin{equation}
\Rinnoutood - R^* < c \sigma^2.
\end{equation}
Taking ratios and dividing by suitable constants we get the desired result.
\end{proof}

\section{Experimental details}
\label{app:experiment-details}
\subsection{CelebA}
\label{app:celeba}

For the results in Table~\ref{table:main}, we used 7 auxiliary binary attributes included in the CelebA dataset: \verb!['Bald', 'Bangs', 'Mustache', 'Smiling', '5_o_Clock_Shadow',!
\verb!'Oval_Face', 'Heavy_Makeup']!.
These attributes tend to be fairly robust to our distribution shift (not hat vs. hat) --- if the person has a 5 o'clock shadow, the person is likely a man.
We use a subset of the CelebA dataset with 2000 labeled examples, 30k in-distribution unlabeled examples, 3000 OOD unlabeled examples, and 1000 validation, in-distribution test, and OOD test examples each.
We report numbers averaged over 5 trials, where on each trial, the in-distribution labeled / unlabeled examples are randomly re-sampled while the validation and test sets are fixed.
The backbone for all models is a ResNet-18~\citep{he2016resnet} which takes a CelebA image downsized to $64 \times 64$ and outputs a binary gender prediction.
All models are trained for 25 epochs using SGD with cosine learning rate decay, initial learning rate 0.1, and early stopped with an in-distribution validation set.
The gender ratios in the in-distribution and OOD set are balanced to 50-50.

\paragraph{Aux-inputs model.}
We incorporate the auxiliary inputs by first training a baseline model $\fbs$ from images to output logit, then training a logistic regression model on the concatenated features $[\fbs(x); z]$ where $z$ are the auxiliary inputs.
We sweep over L2 regularization hyperparameters $C=[0.1, 0.5, 1.0, 5.0, 10.0, 20.0, 50.0]$ and choose the best with respect to an in-distribution validation set.

\paragraph{Aux-outputs model.}
During pretraining, the model trains on the 7-way binary classification task of predicting the auxiliary information. Then, the model is finetuned on the gender classification task without auxiliary information.

\paragraph{In-N-Out and repeated self-training.}
For In-N-Out models with repeated self-training, we pseudolabeled all the unlabeled data using the In-N-Out model and did one round of additional self-training.
Following~\citep{kumar2020gradual}, we employ additional regularization when doing self training by adding dropout with probability 0.8.
We also reduced the learning rate to 0.05 to improve the training dynamics.

\paragraph{Adding auxiliary inputs one-by-one.}
In Figure~\ref{fig:celeba-chain}, we generate a random sequence of 15 auxiliary inputs and add them one-by-one to the model, retraining with every new configuration.
We use the following auxiliary information:
\verb|'Young'|, \verb|'Straight_Hair'|, \verb|'Narrow_Eyes'|, \verb|'Mouth_Slightly_Open'|, \verb|'Blond_Hair'|, \verb|'5_o_Clock_Shadow'|, \verb|'Big_Nose'|, \verb|'Oval_Face'|, \verb|'Chubby'|, \verb|'Attractive'|, \verb|'Blurry'|, \verb|'Goatee'|, \verb|'Heavy_Makeup'|, \verb|'Wearing_Necklace'|, and \verb|'Bushy_Eyebrows'|.

\paragraph{Correlation between in-distribution and OOD accuracy.}
In Figure~\ref{fig:celeba-correlation}, we sample 100 random sets of auxiliary inputs of sizes 1 to 15 and train 100 different aux-inputs models using these auxiliary inputs.
We plot the in-distribution and OOD accuracy for each model, showing that there is a significant correlation between in-distribution and OOD accuracy in CelebA, supporting results on standard datasets~\citep{recht2019doimagenet,xie2020selftraining,santurkar2020breeds}.
Each point in the plot is an averaged result over 5 trials.

\subsection{Cropland}
\label{app:cropland}
All models reported in Table~\ref{table:main} were trained using the Adam optimizer with learning rate $0.001$, a batch size of 256, and 100 epochs unless otherwise specified. Our dataset consists of about 7k labeled examples, 170k unlabeled examples (with 130k in-distribution examples), 7.5k examples each for validation and in-distribution test, and 4260 OOD test examples (the specification of OOD points is described in further detail below). Results are reported over 5 trials, and $\lambda\in \{0.5,0.6,0.7,0.8,0.9\}$ was chosen using the validation set.
On each trial, the in-distribution labeled / unlabeled examples are randomly re-sampled while the validation and test sets are fixed.

\paragraph{Problem Motivation.}
Developing machine learning models trained on remote sensing data is currently a popular line of work for practical problems such as typhoon rainfall estimation, monitoring reservoir water quality, and soil moisture estimation \citep{lary2016MLremotesensing, maxwell2018MLremotesensing, ahmad2010soilmoisture}.  Models that could use remote sensing data to accurately forecast crop yields or estimate the density of regions dedicated to growing crops would be invaluable in important tasks like estimating a developing nation's food security~\citep{li2007estimating}.

\paragraph{OOD Split.}
In remote sensing problems it is often the case that certain regions lack labeled data (e.g., due to a lack of human power to gather the labels on site), so extrapolation to these unlabeled regions is necessary. To simulate this data regime, we use the provided (lat, lon) pairs of each data point to split the dataset into labeled (in-distribution) and unlabeled (out-of-distribution) portions. Specifically, we take all points lying in Iowa, Missouri, and Illinois as our ID points and use all points within Indiana and Kentucky as our OOD set.

\paragraph{Shape of auxiliary info.}
To account for the discrepancy in shapes of the two sources of auxiliary information (latitude and longitude are two scalar measurements while the 3 vegetation bands form a $3 \times 50 \times 50$ tensor), we create latitude and longitude ``bands'' consisting of two $50 \times 50$ matrices that repeat the latitude and longitude measurement, respectively. Concatenating the vegetation bands and these two pseudo-bands together gives us an overall auxiliary dimension of $5 \times 50 \times 50$.

\paragraph{UNet.}
Since our auxiliary information takes the form of $50 \times 50$ bands, we need a model architecture that can reconstruct these bands in order to implement the aux-outputs and the In-N-Out models. With this in mind, we utilize a similar UNet architecture that~\citet{wang2020weakly} use on the same Cropland dataset. While the UNet was originally proposed by~\citet{ronneberger2015unet} for image segmentation, it can be easily modified to perform image-to-image translation. In particular, we remove the final $1 \times 1$ convolutional layer and sigmoid activation that was intended for binary segmentation and replace them with a single convolutional layer whose output dimension matches that of the auxiliary information. In our case, the last convolutional layer has an output dimension of 5 to reconstruct the 3 vegetation bands and (lat,lon) coordinates.

To perform image-level binary classification with the UNet, we also replace the final $1 \times 1$ convolutional layer and sigmoid activation, this time with a global average pool and a single linear layer with an output dimension of 1. During training we apply a sigmoid activation to this linear layer's output to produce a binary class probability, which is then fed into the binary cross entropy loss function.

\paragraph{Aux-inputs model.}
Since the original RGB input image is $3 \times 50 \times 50$, we can simply concatenate the auxiliary info alongside the original image to produce an input of dimensions $8 \times 50 \times 50$ to feed into the UNet.

\paragraph{Aux-outputs model.}
The modification of the traditional UNet architecture in order to support auxiliary outputs for Cropland is described in the above UNet section. We additional add a $\tanh$ activation function to squeeze the model's output values to the range $[-1, 1]$ (the same range as the images). We train the model to learn the auxiliary bands via pixel-wise regression using the mean squared error loss.

\paragraph{In-N-Out model.}
We found that the finetuning phase of the In-N-Out algorithm experienced wild fluctuations in loss and would not converge when using the hyperparameters listed at the top of this section. To encourage the model to converge and fit the training set, we decreased the Adam learning rate to 0.0001 and doubled the batch size to 512.

\paragraph{Repeated self-training.}
For the additional round of self-training, we initialize training and pseudolabel all unlabeled data with the In-N-Out model. Following~\citep{kumar2020gradual}, we employ additional regularization when doing self training by adding dropout with probability 0.8.

\subsection{Landcover}
\label{app:landcover}

Our Landcover dataset comes from NASA's MODIS Surface Reflectance product, which is made up of measurements from around the globe taken by the Terra satellite~\citep{modis2015landcover}. 
In each trial, we use about 16k labeled examples from non-African locations, 203k unlabeled examples (with 150k in-distribution examples), 9266 examples each for validation and in-distribution test, and 4552 OOD test examples.
We trained with SGD + momentum (0.9) on all models for 400 epochs with a cosine learning rate schedule. We used learning rate 0.1 for all models that were not pre-trained, and learning rate 0.01 for models that were already pre-trained.
Results are reported over 5 trials, and $\lambda\in \{0.1, 0.3, 0.5, 0.7, 0.9\}$ was chosen using the validation set.
On each trial, the in-distribution labeled / unlabeled examples are randomly re-sampled while the validation and test sets are fixed.

\paragraph{1D CNN}
While Convolutional Neural Networks are most commonly associated with the groundbreaking success of 2D-CNNs on image-based tasks, the 1-dimensional counterparts have also found success in various applications~\citep{kiranyaz2019cnn1d}. Because the measurements from the MODIS satellite are not images but instead scalar-valued time series data, we can use a 1D CNN with 7 channels, one for each of the 7 MODIS sensors.

\paragraph{NDVI}
The normalized difference vegetation index (NDVI) is a remote sensing measurement indicating the presense of live green vegetation. It has been shown to be a useful predictor in landcover-related tasks~\citep{defries1994ndvi,defries1995AVHRR,lunetta2006land}, so we choose to include it in our models as well. NDVI can be computed from the RED and NIR bands of the MODIS sensors via the equation 
\begin{equation}
	\text{NDVI} = (\text{NIR} - \text{RED}) / (\text{NIR} + \text{RED}).
\end{equation}

We include NDVI along with the 7 other MODIS bands to give us input dimensions of $46 \times 8$.

\paragraph{ERA5}
It is a reasonable hypothesis that having additional climate variables such as soil type or precipitation could be useful for a model in inferring the underlying landcover class. To this end we incorporate features from the ERA5 climate dataset as our auxiliary information~\citep{dataset2017era5}. The specific variables we include are soil type, temperature, precipitation rate, precipitation total, solar radiation, and cloud cover. For each MODIS point we find its nearest ERA5 neighbor based on their latitude and longitude in order to pair the datasets together.

The ERA5 measurements are monthly averages, which means the readings are at a different frequency than that of the 8-day MODIS time series. We upsample the ERA5 signal using the \verb|scipy.signal.resample| method, which uses the FFT to convert to the frequency domain, adds extra zeros for upsampling to the desired frequency, and then transforms back into the time domain.

\paragraph{Landcover classes.}
The Landcover dataset has a total of 16 landcover classes, with a large variance in the individual class counts. To ensure our model sees enough examples of each class, we filtered the dataset to include just 6 of the most populous classes: \verb|savannas|, \verb|woody_savannas|, \verb|croplands|, \verb|open_shrublands|, \verb|evergreen_broadleaf_forests|, and \verb|grasslands|.

\paragraph{Aux-inputs model.}
We concatenate the resampled ERA5 readings with the MODIS and NDVI measurements to obtain an input dimension of $46 \times 14$.

\paragraph{Aux-outputs model.}
Rather than predicting the entire ERA5 time series as an auxiliary output, we instead average the 6 climate variables over the time dimension and predict those 6 means as the auxiliary outputs. We use a smaller learning rate of 0.01 for this pre-trainined model.

\paragraph{In-N-Out and Repeated self-training.}
The In-N-Out model initializes its weights from the aux-outputs model and gets pseudolabeled ID unlabeled data from the aux-inputs model. As with aux-outputs, we use a smaller learning rate of 0.01 for this pre-training model.

For the additional round of self-training, we initialize training and pseudolabel all unlabeled data with the In-N-Out model. Following~\citep{kumar2020gradual}, we employ additional regularization when doing self training by adding dropout with probability 0.5. We found that with dropout, we need a higher learning rate (0.1) to effectively fit the training set.

\end{document}